\documentclass{article}

\usepackage[preprint]{clgnn_style}

\usepackage[utf8]{inputenc} 
\usepackage[T1]{fontenc}    
\usepackage{hyperref}       
\usepackage{url}            
\usepackage{booktabs}       
\usepackage{amsfonts}       
\usepackage{nicefrac}       
\usepackage{microtype}      
\usepackage{xcolor}         
\usepackage{amsthm}  
\newtheorem{lemma}{Lemma}  
\newtheorem{proposition}{Proposition}[section]

\usepackage{amsmath} 
\usepackage[most]{tcolorbox}
\usepackage{amsmath}
\usepackage{amssymb}
\usepackage{amsthm}
\usepackage{algorithmic}
\usepackage{graphicx}
\usepackage{pifont}
\theoremstyle{definition}
\usepackage{graphicx}
\usepackage{enumitem}
\usepackage{subcaption} 
\usepackage[linesnumbered,ruled,vlined]{algorithm2e}
\newtheorem{definition}{Definition}

\title{CLGNN: A Contrastive Learning-based GNN Model for Betweenness Centrality Prediction on Temporal Graphs}

\begin{document}

\maketitle
\vspace{-40pt}
\begin{center}
\textbf{Tianming Zhang\textsuperscript{a,*}, Renbo Zhang\textsuperscript{b,d,*}, Zhengyi Yang\textsuperscript{c}, Yunjun Gao\textsuperscript{d}, Bin Cao\textsuperscript{a}, Jing Fan\textsuperscript{a}}\\[1em]

\textsuperscript{a}School of Computer Science and Technology, Zhejiang University of Technology, 288 Liuhe Road, Hangzhou, 310023, Zhejiang, China\\
\textsuperscript{b}Center for Data Science, Zhejiang University, 866 Yuhangtang Road, Hangzhou, 310030, Zhejiang, China\\
\textsuperscript{c}School of Computer Science and Engineering, University of New South Wales, Sydney, 2052, New South Wales, Australia\\
\textsuperscript{d}School of Computer Science and Technology, Zhejiang University, 38 Zheda Road, Hangzhou, 310013, Zhejiang, China\\[1em]

\textsuperscript{*}Equal contribution.
\end{center}

\begin{abstract}
  Temporal Betweenness Centrality (TBC) measures how often a node appears on optimal temporal paths, reflecting its importance in temporal networks. However, exact computation is highly expensive, and real-world TBC distributions are extremely imbalanced. The severe imbalance leads learning-based models to overfit to zero-centrality nodes, resulting in inaccurate TBC predictions and failure to identify truly central nodes. Existing graph neural network (GNN) methods either fail to handle such imbalance or ignore temporal dependencies altogether. To address these issues, we propose a scalable and inductive contrastive learning-based GNN (CLGNN) for accurate and efficient TBC prediction. CLGNN builds an instance graph to preserve path validity and temporal order, then encodes structural and temporal features using dual aggregation, i.e., mean and edge-to-node multi-head attention mechanisms, enhanced by temporal path count and time encodings. A stability-based clustering-guided contrastive module (KContrastNet) is introduced to separate high-, median-, and low-centrality nodes in representation space, mitigating class imbalance, while a regression module (ValueNet) estimates TBC values. CLGNN also supports multiple optimal path definitions to accommodate diverse temporal semantics. Extensive experiments demonstrate the effectiveness and efficiency of CLGNN across diverse benchmarks. CLGNN achieves up to a 663.7~$\times$ speedup compared to state-of-the-art exact TBC computation methods. It outperforms leading static GNN baselines with up to 31.4~$\times$ lower MAE and 16.7~$\times$ higher Spearman correlation, and surpasses state-of-the-art temporal GNNs with up to 5.7~$\times$ lower MAE and 3.9~$\times$ higher Spearman correlation.
\end{abstract}
\section{Introduction}
\label{sec:intro}
Real-world systems such as social networks 
and communication networks are often modeled as temporal graphs, where edges are associated with time intervals or discrete time points, indicating when they are active. In such graphs, certain individuals play a crucial role in driving rapid information diffusion or facilitating rumor spreading. Temporal generalizations of betweenness centrality are defined based on the optimal temporal paths. However, a significant challenge lies in the computational complexity of these generalizations, as calculation requires evaluating all possible temporal paths, which is resource-intensive. 
To address this, we propose a novel graph neural network (GNN) model designed to accurately predict Temporal Betweenness Centrality (TBC) values. 

Existing graph learning models primarily focus on either computing betweenness centrality (BC) for static graphs \cite{FanZDCSL19, MauryaLM19} or performing vertex ranking on temporal graphs \cite{10.1145/3696410.3714943,ZhangF00F24}. To the best of our knowledge, only DBGNN \cite{HeegS24} predicts TBC values on temporal graphs. However, its prediction accuracy is low, our experiments show that the mean absolute error (MAE) for TBC reaches 496 on the dataset \textit{Highschool2013}. Unlike conventional ranking tasks, predicting TBC values on temporal graphs is more challenging due to the extreme class imbalance problem.
\begin{figure*}[t]
    \centering
    \begin{subfigure}[t]{0.4\linewidth}
        \centering
        \includegraphics[width=\linewidth]{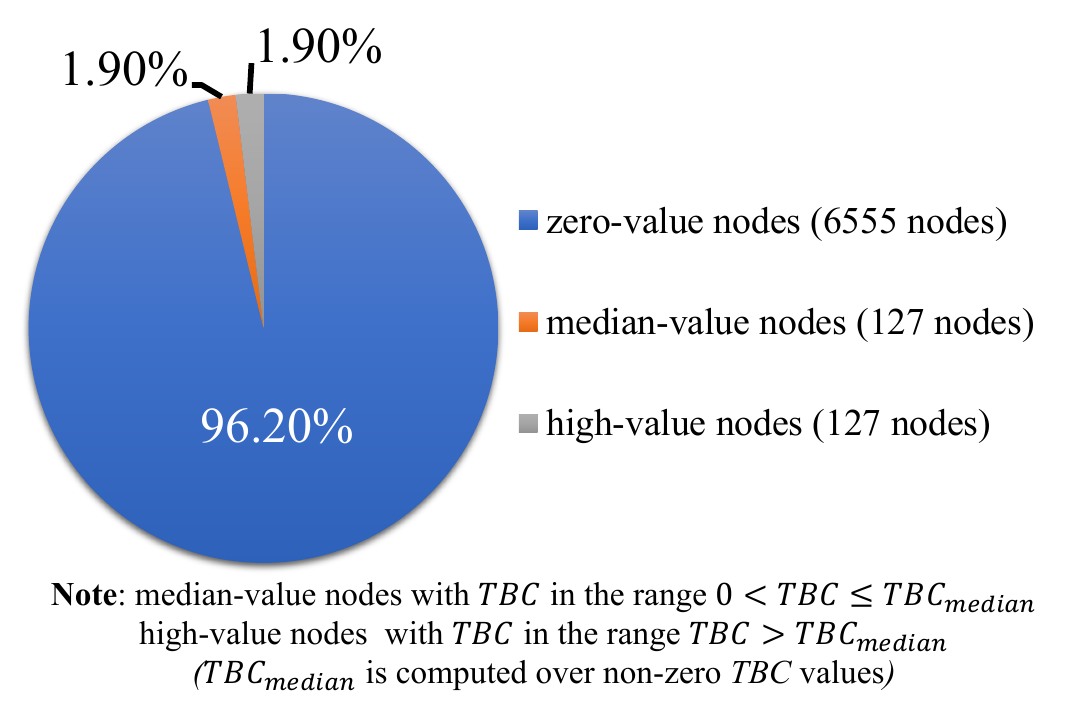}
        \caption{The proportion of nodes in 3 categories}
        \label{fig:example1}
    \end{subfigure}%
    \hspace{0.05\linewidth}
    \begin{subfigure}[t]{0.4\linewidth}
        \centering
        \includegraphics[width=\linewidth]{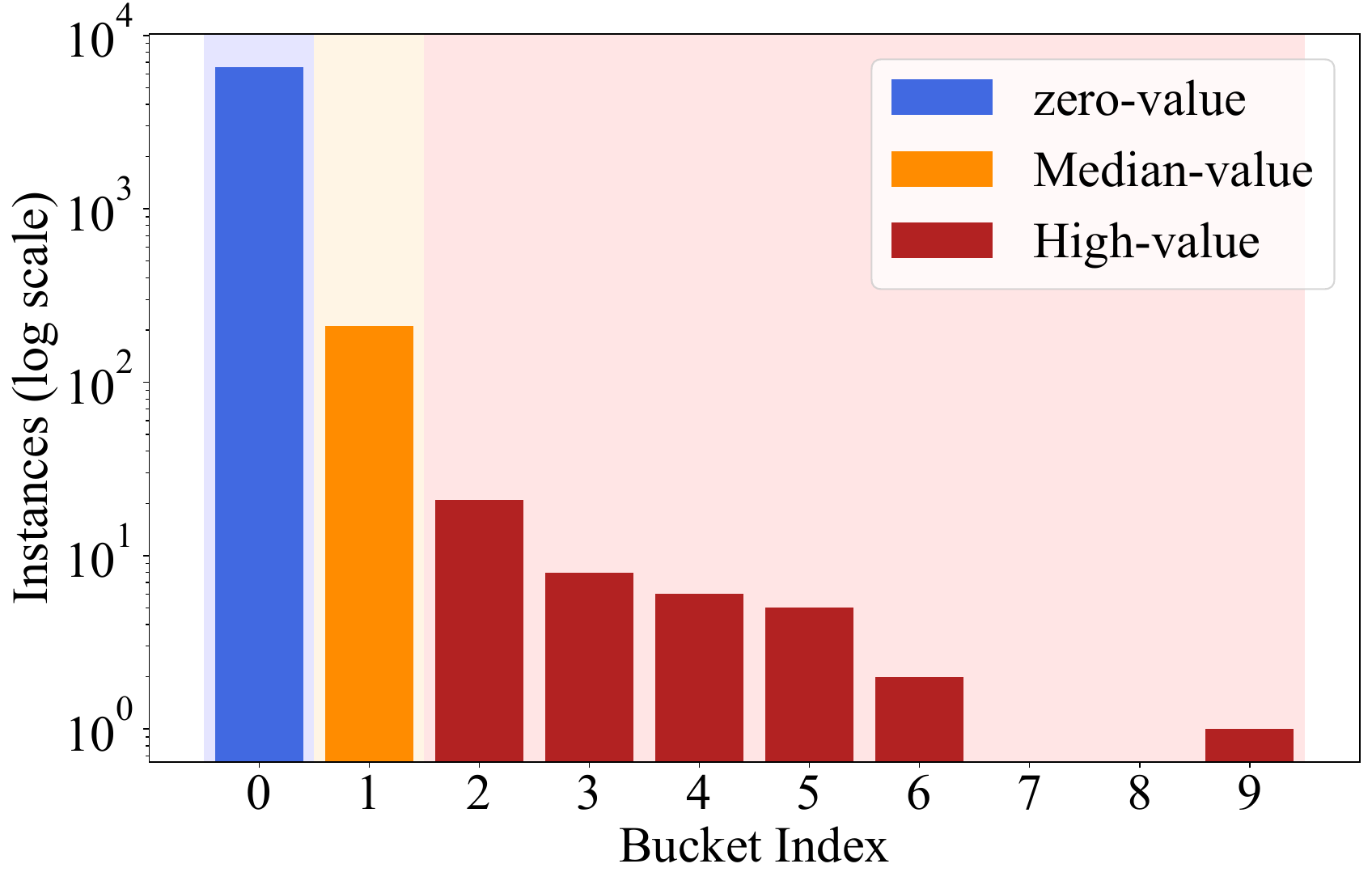}
        \caption{\mbox{The distribution of TBC values across 10 buckets}}
        \label{fig:example2}
    \end{subfigure}
    \vspace{-0.05in}
    \caption{Distribution of TBC values in the \textit{ia-reality-call} dataset}
    \label{fig:dataset_example}
    \vspace{-0.29in}
\end{figure*}

To illustrate this, we conducted a statistical analysis of the distribution of TBC values in the \textit{ia-reality-call} dataset, as shown in Figure~\ref{fig:dataset_example}. Figure~\ref{fig:example1} groups nodes into three categories (i.e., zero-, median-, and high-value) based on their TBC scores. The results reveal a highly imbalanced distribution, with 96.2\% of nodes having zero TBC, while only 1.9\% fall into each of the median and high-value categories.
Figure~\ref{fig:example2} presents a log-scale histogram of TBC values. The non-zero range is evenly divided into 9 equal-width intervals,  with a separate bucket for zero values.
It further demonstrates that high TBC nodes are not only rare but also span a wide range of magnitudes. 
This highlights the highly imbalanced nature of TBC distribution, where median and high values have significantly fewer observations.  This imbalance causes two problems:
(1) \textbf{Training is inefficient as most TBC values of vertices are zero, which contribute no useful learning for critical nodes with large TBC values}; (2)  \textbf{the zero values can overwhelm training and lead to degenerate models}. 

In such circumstances, applying existing temporal graph models (e.g., TATKC \cite{ZhangF00F24}, DBGNN~\cite{HeegS24}) results in significant limitations. These models fail to adequately learn the true TBC values of vertices in the median and high-value intervals, leading to most predicted TBC values in these categories being erroneously classified as 0. 
Existing work~\cite{YangZCWK21} has shown that trained models tend to be biased toward head classes (e.g., the zero-value class) with massive training data, resulting in poor performance on tail classes (e.g., median- and high-value intervals) that have limited data.
This is particularly problematic as vertices in the median and high-value intervals are critical hubs for analyzing information flow and propagation. 
\textbf{We identify TBC distribution imbalance during training as the main obstacle impeding the model from achieving good accuracy.}

Directly applying techniques from imbalanced text and image learning to graph data, such as oversampling, undersampling, loss re-weighting, and data augmentation (e.g., node dropping, edge perturbation, and feature masking), has proven ineffective, 
as demonstrated in the experimental results shown in Tables~\ref{tab:ablation_contrastive}. Unlike pixels in images, graph nodes are inherently interconnected, and modifying the topology directly alters the TBC values, thereby distorting the underlying data distribution.
Inspired by the intuition behind contrastive learning, which aims to identify shared features among similar instances while distinguishing features among dissimilar ones, we propose applying this paradigm to TBC prediction. Our objective is to improve the accuracy of predictions for medium and high TBC values, particularly by avoiding their misclassification as zero-value nodes. 
By incorporating contrastive learning, the representations of zero-value vertices can be made more similar, while the representations of median and high-value nodes are positioned farther away from those of zero-value nodes. This separation allows the model to better capture and distinguish the unique features of median and high-value nodes, ultimately improving the accuracy of predictions across all TBC value ranges. 

Based on these insights, we propose CLGNN, a \textbf{C}ontrastive \textbf{L}earning-based \textbf{GNN} model composed of representation and prediction modules. The representation module aims to learn path-time dual-aware vertex representations that effectively capture both path-based and temporal dependencies. 
We construct an instance graph where each node aggregates all timestamps from its outgoing edges, and maintain reachable temporal path count lists that record valid temporal path dependencies while filtering out irrelevant paths. We then design a path count encoding, which, together with time encoding, is fused into a dual aggregation mechanism (i.e., mean aggregation and edge-to-node multi-head attention) to encode structural and temporal characteristics into rich node embeddings.

In the prediction module, we design two complementary components, KContrastNet and ValueNet. KContrastNet adopts a contrastive learning strategy to differentiate features among low, medium, and high-value nodes. To improve the selection of positive and negative sample pairs, we first perform clustering and then select positive pairs as nodes with similar TBC values within the same cluster. Conversely, nodes with significantly different TBC values within the same cluster are treated as negative pairs. ValueNet utilizes multi-layer perceptrons (MLPs) to estimate temporal betweenness centrality (TBC) values. To summarize, our main contributions are as follows:

\begin{itemize}[leftmargin=10pt, topsep=0pt, itemsep=1pt, parsep=0pt, partopsep=0pt]
\item We first propose an inductive and scalable contrastive learning-based TBC regression model that is trained on small graphs and generalizes to unseen graphs. In contrast to the state-of-the-art method DBGNN~\cite{HeegS24}, which trains and tests on the same graph and thereby limiting its generality, our model supports cross-graph inference.

\item The proposed model, CLGNN, jointly learns path-time dual-aware node representations and refines node embeddings via an improved stability-based clustering-guided contrastive learning to enhance TBC prediction.


\item  We conduct extensive experiments on 12 real datasets of varying sizes to evaluate CLGNN. CLGNN delivers up to 663.7~$\times$ faster performance than leading exact TBC computation methods. It also achieves up to 31.4~$\times$ lower MAE and up to 16.7~$\times$ higher Spearman correlation compared to top static GNN baselines, and surpasses state-of-the-art temporal GNNs with up to 5.7~$\times$ lower MAE and 3.9~$\times$ higher Spearman correlation.

\end{itemize}



\section{Related work}
\label{sec:relate}

\textbf{Traditional BC computation method.} 
Representatives for exact BC computation on static graphs include  Brandes' algorithm~\cite{brandes2001faster} and its variants, such as Brandes++~\cite{ErdosIBT15} and BADIOS~\cite{SariyuceKSC17}. Approximate methods leveraging techniques such as source sampling~\cite{BrandesP07,JacobKLPT04} and node-pair sampling~\cite{RiondatoK16,RiondatoU16,BorassiN19,CousinsWR21,LF21} to improve accuracy and efficiency. Notable examples include RK~\cite{RiondatoK16}, KADABRA~\cite{BorassiN19}, and ABRA~\cite{RiondatoU16}, which offer theoretical guarantees. 
Variants like $k$-betweenness centrality~\cite{PfefferC12}, $\kappa$-path centrality~\cite{KourtellisASIT17}, top-$k$ ego BC~\cite{ZhangLPDWY22} , Coarse-grained and Fine-grained BC~\cite{Wang24} aim to simplify computation. In temporal graphs, 
research has focused on temporal paths~\cite{BussMNR20,ERP15} and BC variants. 
Tsalouchidou et al.~\cite{TsalouchidouBBL20}  explored exact computation using static snapshots or predecessor graphs.  Methods like ONBRA~\cite{SantoroS22} have introduced approximation techniques with probabilistic guarantees. Zhang et al.~\cite{TBC24}  proposed exact and approximate methods based on transformed time instance graphs. Regunta et al.~\cite{ReguntaTSK23}  leveraged parallelism to update TBC values incrementally.  Although the above-mentioned optimization techniques have been proposed, the computational complexity of TBC is still relatively high for large-scale temporal graphs.

\textbf{Graph neural network based model.}  
The modeling of dynamic graphs primarily focuses on leveraging temporal neural networks to capture evolving structures. Early approaches extended static  GNNs by incorporating RNNs to capture temporal dependencies~\cite{pareja2020evolvegcn, sankar2020dynamic}. More recent methods, such as TGAT~\cite{xu2020inductive}, TGN~\cite{rossi2020temporal}, and GATv2~\cite{brody2022attentive}, integrate self-attention mechanisms.  
De Bruijn Neural Networks~\cite{debruijn2023causality} extend causality-aware mechanisms for temporal graphs.  Nevertheless, these models failed to capture the intricate path structures required for betweenness computations, leading to low estimation accuracy. Some models are specially designed for BC ranking or value prediction on static graphs.  Early studies~\cite{MendoncaBZ21} used ML, while more recent approaches~\cite{FanZDCSL19, MauryaLM21} exploit graph attention and reinforcement learning to approximate centrality. CNCA-IGE~\cite{Zou24}  adopts an encoder-decoder for centrality ranking. In temporal settings, TATKC~\cite{ZhangF00F24} uses time-injected attention for Katz centrality ranking, and DBGNN~\cite{HeegS24} predicts centrality with a time-aware GNN. While effective in ranking, these models struggle with precise TBC value prediction, and they often misclassify mid-to-high value nodes as zero. Additionally, DBGNN is transductive, limiting generalization to unseen networks.

\section{Preliminaries}
\label{sec:preli}

\begin{definition}\label{defn:tg}
\textbf{$($Temporal Graph$)$.} A directed temporal graph is denoted by $G=(V, E, T)$, where $($i$)$ $V$ is a vertex set; $($ii$)$ $E \in V \times V \times T$ is the set of directed time-stamped edges;
and $($iii$)$ $T$ is the set of discrete or continuous time points at which interactions occur. Each edge $e = (u, v,t) \in E$ represents a connection from node $u$ to $v$ in $V$ at a specific time $t \in T$. 

\end{definition}




\begin{definition}\label{defn:tp}
\textbf{$($Temporal Path$)$.} A temporal path in  $G=(V,E,T)$ is a sequence of edges $p$ $=$ $u$ $\stackrel{t_1}{\longrightarrow}$ $w_1$ $\cdots$ $\stackrel{t_{m-1}}{\longrightarrow}$ $w_{m-1}$ $\stackrel{t_m}{\longrightarrow}$ $v$, where (i) each edge in $p$ belongs to the edge set $E$; (ii) the sequence of timestamps satisfied $0$ $<$ $t_{i+1}$$-$$t_i$ $\le$ $\delta$  for  $1 \le i < m$, where $\delta \in \mathbb{R}$, ensuring that the interactions not only follow a ascending temporal order but also occur within a bounded time interval. 
\end{definition}
Temporal graphs differ from static graphs due to the presence of temporal constraints, which lead to multiple notions of optimal paths beyond the standard shortest path. Common definitions include the shortest, earliest arrival, and latest departure temporal paths, each optimizing for different temporal criteria. Given a temporal graph $G=(V,E,T)$, a \textbf{shortest temporal path} from source $s$ to destination $d$ is a temporal path (not necessarily unique) that contains the minimum number of hops (edges). An \textbf{earliest arrival temporal path}  reaches $d$ as early as possible. An \textbf{latest departure path} leaves $s$ as late as possible while still reaching 
$d$ on time. Based on optimal paths, the formal definition of temporal betweenness centrality (TBC) ~\cite{TsalouchidouBBL20, TBC24} is as follows:

\begin{definition}\label{defn:tbc}
{\bf (Temporal Betweenness Centrality, TBC)}. Given a temporal graph $G=(V,E,T)$, the normalized temporal betweenness centrality value TBC($v$) of a node $v \in V$ is defined as:
\begin{equation}\nonumber
\begin{split}
TBC(v)= \frac{1}{|V|(|V|-1)}\sum_{\forall s \neq v \neq z \in V}\frac{\sigma_{sz}(v)}{\sigma_{sz}},
\end{split}
\end{equation}
where $|V|$ is the total number of vertices in $G$; $\sigma_{sz}$ denotes the number of optimal temporal paths from $s$ to $z$; and $\sigma_{sz}(v)$ represents the number of such paths pass though $v$. 
\end{definition}



\section{proposed model}
\label{sec:proposedmodel}

\begin{figure*}[htbp]  
\vspace{-16pt}
  \centering  
  \includegraphics[width=\textwidth]{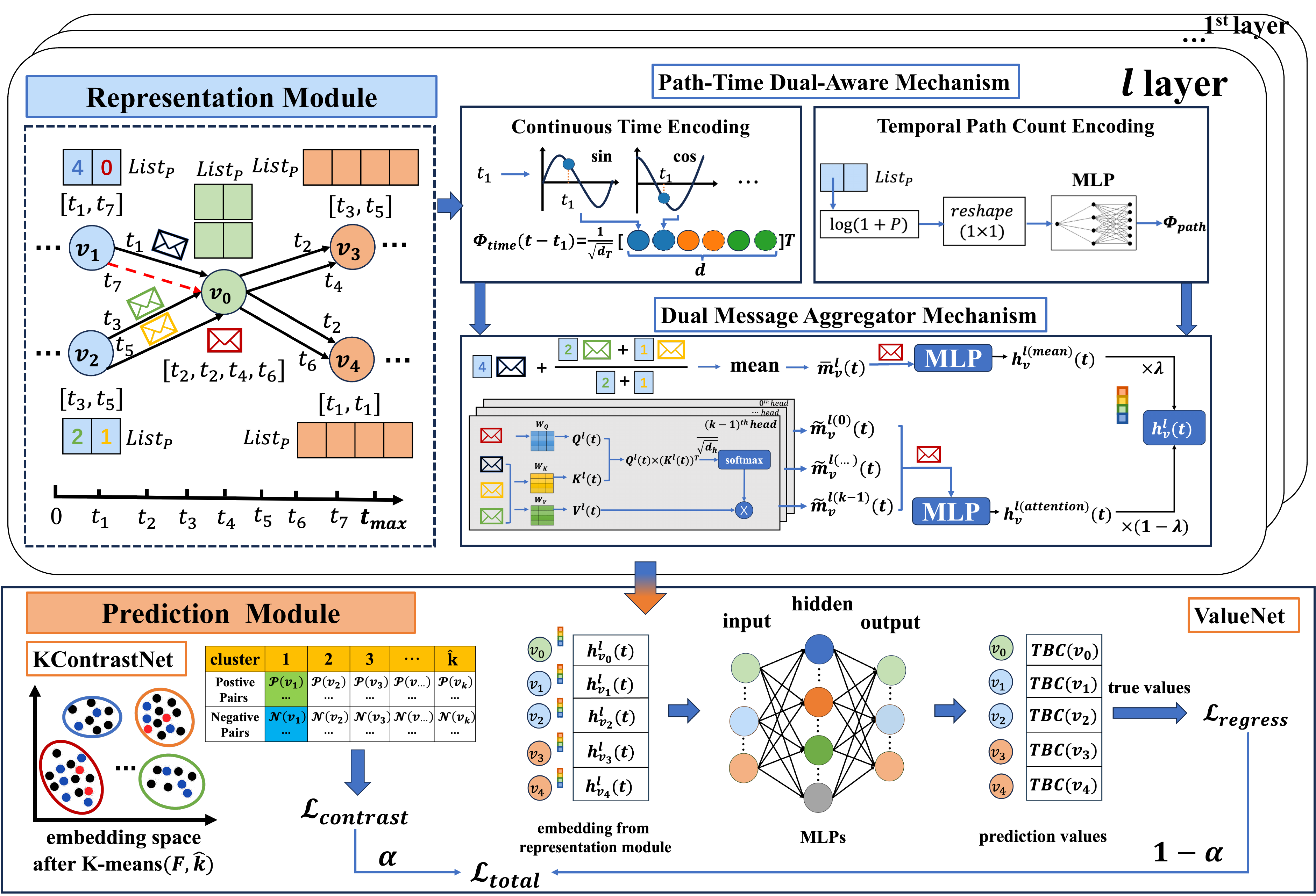}  
  \caption{Overview of the Proposed Model: CLGNN}  
  \label{fig:CLGNNarchitecture}  
\vspace{-10pt}
\end{figure*}
To effectively predict TBC under strong data imbalance, we propose a contrastive learning-based GNN model, CLGNN, as illustrated in Figure~\ref{fig:CLGNNarchitecture}. CLGNN is composed of two main components, i.e., a representation module and a prediction module. We describe each module in detail below. 

\subsection{Representation Module}\label{sec:rep}
To capture the temporal paths and preserve the temporal dependencies, we first transform temporal graphs into an instance-level structure, where each node $v$ is associated with the set $T_{out}(v)$ of the timestamps attached with $v$'s outgoing edges. In addition, for each node $u$ and its outgoing edge $(u,v,t)$, we maintain the number of valid temporal paths, denoted as:
\[
P(u,v,t) = \sum_{t_v \in T_{\mathrm{out}}(v)} \mathbb{I}(t_v-t),
\]
where $\mathbb{I}(x)$ is an indicator function that defined as: $\mathbb{I}(x) = 
\begin{cases}
1, & \text{if } x > 0 \\
0, & \text{otherwise}
\end{cases}$. If $P(u,v,t)=0$, then the message generated by node $u$ will not be transmitted to node $v$ through the edge with timestamp $t$ during the message passing process. 

\textbf{Path-Time dual-aware mechanism.}
It combines continuous time encoding and temporal path count encoding, jointly capturing temporal dependencies and relevant path information in the message passing process.

\textbf{\textbullet~Continuous Time Encoding.} We employed the approach in \cite{xu2020inductive}, which applies Bochner's theorem and Monte Carlo integral to map a temporal domain \( T = [0, t_{\text{max}}] \) to the \( d \)-dimensional vector space. 
\[
\Phi_{time}(t) = \sqrt{\frac{1}{d_T}} \left[ \cos(\omega_1 t), \sin(\omega_1 t), \dots, \cos(\omega_{d_T} t), \sin(\omega_{d_T} t) \right]^T ,
\]
where \( d_T \) is the finite dimension and \( \omega = (\omega_1, \dots, \omega_{d_T})^T \) are learnable parameters.

\textbf{\textbullet~Temporal Path Count Encoding.}
Temporal path count encoding processes the number of temporal paths to provide stable and efficient input features for the neural network. It is defined as: 
\[
P(u, v, t) \xrightarrow{\text{log-transformed value}} \log(1 + P(u, v, t)),
\]
\[
\Phi_{path}(u,v,t) = \text{MLP}(\text{reshape}( \log(1 + P(u,v,t)) )),
\]
where the reshaping operation \( \text{reshape}(1, 1) \) converts the scalar into a suitable format for input into the network. Note that the path count \( P(u, v, t) \) undergoes a \( \log(1 + P) \) transformation, which aims to compress the influence of large path values on model training, alleviating gradient explosion issues while preserving non-negativity and the distinguishability of the values.

\textbf{Message Function.}
For each interaction event $e=(u,v,t)$, a message is computed to update the representation of node $v$. Specifically, the message at layer $l$ is defined as:
\[
m^l_{uv}(t) = \left( h^{l-1}_{u}(t^-) || h^{l-1}_{v}(t^-)||\Phi_{time}(t-t^-) || \Phi_{path}(u,v,t) \right),
\]

where $h^{l-1}_{u}(t^-)$, $h^{l-1}_{v}(t^-) \in \mathbb{R}^{d}$ are the hidden representation for node \( u \) and \( v \) in the previous layer (the $l-1$ layer) of the neural network just before time \( t \) (i.e., from the time of the previous interaction involving \( u \) and \( v \)), and $||$ denotes concatenation. 
The computed message $m^l_{uv}(t)$ is then passed into the aggregation function and contributes to the update of node $v$'s embedding at layer $l$.

\textbf{Dual Message Aggregator Mechanism.}
Since each node typically interacts with multiple neighbors and may receive multiple messages from the same neighbor at different times, there exist two levels of message aggregation: across edges and across nodes. To address this, we employ a dual aggregation mechanism, i.e., mean aggregation and edge-to-node multi-head attention, to balance the model’s stability and flexibility.

\textbf{\textbullet~Mean aggregation (node-level).} 
Mean aggregation performs aggregation at the node level, ensuring that all neighbors contribute equally to the target node’s representation. 
At the $l$-th layer, the mean-aggregated message received by node \( v \) at time \( t \) is computed as:
\[
\bar{h}^l_v(t) = \frac{1}{|N_{\text{in}}(v)|} \sum_{u \in N_{\text{in}}(v)} \left( \sum_{t_x \in T_{uv}}  \frac{P(u, v, t_x)}{\sum_{t_x' \in {T}_{uv}} P(u, v, t_x')}  \cdot m^l_{uv}(t_x) \right),
\]
where $N_{in}(v)$ is the set of $v$'s in-neighbors; $T_{uv}$ is the set of timestamps associated with all edges from $u$ to $v$; \( m^l_{uv}(t_x) \) is the message computed at time \( t_x \) in the $l$-th layer. 
Notably, only temporal paths that are reachable (i.e., those with \( P(u, v, t_x) > 0 \)) are considered in the aggregation. 

\textbf{\textbullet~Edge-to-node multi-head attention (edge-level).}
This mechanism computes attention for each edge connected to a node, ensuring that all temporal edges contribute to the node's representation update.  
The input matrix of the $l$-th attention layer is:
\[
M_v^l(t) = \left[ m^l_{uv}(t_x) \right]_{\forall(u,v, t_x) \in E_{in}(v),P(u, v, t_x) > 0 }^T, 
\]
$E_{in}(v)$ is the set of $v$'s incoming edges. 
Then \( M^l(t) \) is projected into query, key, and value spaces: $Q^l(t) = [h^{l-1}_{v}(t^-) \,||\, h^{l-1}_{v}(t^-) \,||\, \Phi_{\text{time}}(t - t^-) \,||\, \Phi_{\text{path}}(u, v, t)] \cdot W_Q$; $K^l(t) = M^l(t) \cdot W_K$; $V^l(t) = M^l(t) \cdot W_V$, where \( W_Q, W_K, W_V \in \mathbb{R}^{(d + d_T) \times d_h} \) are learnable projection matrices, and \( d_h \) is the hidden dimension. 
Next, the scaled dot-product attention is used in attention layers.  The aggregated message for node $v$ at the $l$-th layer, computed via the multi-head attention mechanism, is given by:
\[
\tilde{h}_{v}^l(t)=\text{Attn}(Q^l(t), K^l(t), V^l(t)) = \text{softmax} \left( \frac{Q^l(t)(K^l(t))^T}{\sqrt{d_h}} \right) V^l(t) \in \mathbb{R}^{d_h}
\]
\textbf{Embedding.}
The final embedding of node \( v \) at time \( t \), denoted as \( h^l_v(t) \),  is obtained by combining the outputs of mean aggregation and edge-to-node multi-head attention:
\begin{align*}
h^l_v(t) &= \lambda \cdot h_v^{l(\text{mean})}(t) + (1 - \lambda) \cdot h_v^{l(\text{attention})}(t), \\
h_v^{l(\text{mean})}(t) &= \text{MLP}( \bar{h}_v^{l}(t) \parallel h_v^{l-1}(t)), \\
h_v^{l(\text{attention})}(t) &= \text{MLP}\left( \tilde{h}_v^{l(0)}(t) \parallel \tilde{h}_v^{l(1)}(t) \parallel \cdots \parallel \tilde{h}_v^{l(k-1)}(t) \parallel h_v^{l-1}(t) \right),
\end{align*}
where \( k \) is the number of heads, and \( h_v^{l(j)}(t) \) \((0 \leq j < k)\) is the dot-product attention output of the \( j \)-th head.   \( \lambda \) is a coefficient that balances the contribution of the two aggregation types. This weighted combination allows the model to trade off between the stability of equal neighbor contributions and the flexibility of dynamic attention-based relevance, enabling expressive temporal node representations.

    


 \subsection{Prediction Module}\label{sec:pre}
To address the imbalance in temporal betweenness centrality values and enhance the model's sensitivity to rare but important nodes, we design two complementary components, i.e., KContrastNet and ValueNet, in the prediction module.

\textbf{KContrastNet.}
 KContrastNet employs a contrastive learning strategy to effectively distinguish features among low-, medium-, and high-value nodes in TBC prediction. To enhance the selection of positive and negative sample pairs, we first introduce a stability-based clustering approach that identifies the optimal number \(k\) of clusters in the node representation space. 

\textbf{\textbullet~Stability-based clustering method.}
Traditional methods for choosing the number of clusters, such as the elbow method~\cite{thorndike1953belongs} or silhouette score~\cite{rousseeuw1987silhouettes}, often struggle with noisy data or class imbalance, leading to unstable results. To address this, we propose a stability-based clustering method that combines bootstrap resampling~\cite{2012Selection} with K-means clustering~\cite{hartigan1979algorithm} to estimate the optimal number of clusters. To improve efficiency, we first perform stratified random sampling on the original dataset \( F \) to obtain a reduced subset \( F^n \)  (e.g., 40\% of the original size) that retains the original distribution.  Bootstrap resampling is then applied on \( F^n \) to generate \( B \)  independent sample pairs \( \{(X_i, Y_i)\}_{i=1}^B \), ensuring each drawn independently and identically from the empirical distribution (a formal justification is provided in Appendix~\ref{appendix:iid}). For each pair and each candidate cluster number \( k \in \{2, \dots, K\} \), we run k-means separately on  \( X_i \) and \( Y_i \) to obtain two cluster assignments \( \psi_{X_i}\) and \( \psi_{Y_i}\), respectively. We then compute the clustering distance between these assignments and average the results over all \( B \) pairs to obtain a clustering instability, as formally defined below.
\begin{definition}\label{defn:CD}
{\bf (Clustering Distance)}. The clustering distance between two clustering results $\psi_{X_i}$ and $\psi_{Y_i}$ is defined as:
\[
d(\psi_{X_i}, \psi_{Y_i}) = E_{x, y \in (X_i \cap Y_i)} \left[ \delta(\psi_{X_i}(x), \psi_{X_i}(y)) - \delta(\psi_{Y_i}(x), \psi_{Y_i}(y)) \right],
\]
where \( E \) denotes the expectation operator, computed over all possible samples \( x \) and \( y \). $\psi_{X_i}(x)$ denotes the cluster label assigned to sample $x$ in the clustering result $\psi_{X_i}$. The function $\delta(a, b)$ is the Kronecker Delta, defined as: $
\delta(a, b) = 
\begin{cases}
1, & \text{if } a = b, \\
0, & \text{if } a \neq b.
\end{cases}.
$
\end{definition}
\begin{definition}\label{defn:CI}
{\bf (K-means Clustering Instability)}. 
The clustering instability quantifies the variation in clustering results across different resamples and is defined as:
\[
\hat{s}_B(\psi, k, n) = \frac{1}{B} \sum_{i=1}^{B} d\left(\text{KMeans}(X_{i}, k), \text{KMeans}(Y_{i}, k)\right),
\]
$\psi$ is the clustering results of \text{KMeans} function and 
$d(\cdot,\cdot)$ is the clustering distance defined above.
\end{definition}
Finally, the optimal number of clusters  $\hat{k}$ $=$ $\arg~\min_{2\leq k \leq K}$ $\hat{s}_B$ $(\psi,$ $k,$ $n)$ by minimizing the instability. Statistical properties of the estimator, including unbiasedness and consistency, are detailed in Appendix~\ref{appendix:instability_consistency}. And the complete algorithm and its complexity analysis are provided in Appendix~\ref{appendix:algorithm}.

\textbf{\textbullet~Contrastive Learning.}
After clustering, each node is assigned a cluster label, which guides the construction of supervised contrastive learning pairs.   To ensure that positive samples are both semantically meaningful and structurally consistent, we introduce a conditional sampling strategy. Specifically, positive samples are formed by selecting nodes from the same cluster that have similar TBC values, while negative samples are chosen from the same cluster but with large TBC differences. A theoretical justification for choosing within-cluster sampling alternatives is provided in Appendix~\ref{appendix:intracluster}.

Positive sample selection: let \( \mathcal{S^+}(u) \) be the set of positive samples for node \( u \), defined as:
\[
\mathcal{S}^+(u) = \left\{ v \mid \delta(\psi(u), \psi(v))=1, 0<\left| TBC(u) - TBC(v) \right| \leq \gamma_{pos} \cdot TBC_{median} \right\},
\]
where \( \gamma_{pos} \in (0,1) \) is a similarity threshold  (e.g., 0.5). The use of the median TBC value ensures robustness against outliers and skewed distributions, improving the stability of contrastive thresholds. 

Negative sample selection: let \( \mathcal{S}^-(u) \) be the set of negative samples for node \( u \), defined as:
\[
\mathcal{S}^-(u) = \left\{ v \mid \delta(\psi(u), \psi(v))=1, \left| TBC(u) - TBC(v) \right| \geq \gamma_{neg} \cdot TBC_{median} \right\},
\]
where \( \gamma_{neg} \) (e.g., 0.5) is a dissimilarity threshold.

The contrastive re-weighted loss \( \mathcal{L}_{\text{contrast}} \) is then computed across all clusters, encouraging the embeddings of positive pairs to be closer while pushing negative pairs apart, weighted by their semantic similarity. The loss is defined as follows:
\[
\mathcal{L}_{\text{contrast}} = \sum_{i=1}^{\hat{k}} \sum_{u \in \psi_i} -\log \frac{\sum\limits_{v \in \mathcal{S}^+(u)} \exp\left( \beta_{uv} \cdot \text{sim}(\mathbf{h}_u, \mathbf{h}_v)/\tau \right)}{\sum\limits_{v \in \mathcal{S}^+(u)} \exp\left( \beta_{uv} \cdot \text{sim}(\mathbf{h}_u, \mathbf{h}_v)/\tau \right) + \sum\limits_{w \in \mathcal{S}^-(u)} \exp\left( \beta_{uw} \cdot\text{sim}(\mathbf{h}_u, \mathbf{h}_w)/\tau \right)}
\]
where \(\mathbf{h}_u\) is the node embedding of GNN introduced in section~\ref{sec:rep},  \( \text{sim}(\cdot, \cdot) \) denote the  dot product similarity.  \( \tau \) is the temperature parameter of InfoNCE~\cite{InfoNCE}.  the pairwise weights for positive samples \( \beta_{uv} \) and negative samples \( \beta_{uw} \) are calculated as:
\[
\beta_{uv} = \frac{TBC_{median}\cdot \gamma_{pos}}{\left| TBC(u) - TBC(v) \right|} , ~~~~
\beta_{uw} = \frac{\left| TBC(u) - TBC(w) \right|}{TBC_{median}\cdot \gamma_{neg} } 
\]
The weight design ensures that node embeddings are adjusted in proportion to their TBC differences. For positive samples,  the smaller the TBC gap between nodes $u$ and $v$, the larger the weight \( \beta_{uv} \),  encouraging the model to learn more from highly similar pairs. For negative samples, the weight   \( \beta_{uw} \) increases with larger TBC differences, emphasizing separation from more dissimilar nodes. A theoretical analysis of this contrastive loss design is provided in Appendix~\ref{appendix:contrastive_theory}. 

\textbf{ValueNet.}
For each node \( v \in V \), the TBC score at the largest timestamp (i.e., lateset time) \( t_{\max} \), denoted as \( \text{TBC}_{t_{max}}(v) \),  represents the final TBC value of \( v \) under the full-graph setting. It is predicted using a three-layer MLP with ReLU activations:
\[
TBC(v) = TBC_{t_{\max}}(v) = \text{MLP}(h_v^l(t_{\max})),
\]
where \( h_v^l(t_{max}) \) is the output representation of node \( v \) from the last GNN layer at time \( t_{max} \).

\textbf{Total Loss Function.}
A weighted total loss function that balances the regression loss and the contrastive loss is defined as:
\[
\mathcal{L}_{\text{total}} = \alpha \cdot \mathcal{L}_{\text{contrast}} + (1 - \alpha) \cdot \mathcal{L}_{\text{regress}}
\]
where \( \mathcal{L}_{\text{regress}} \) is the regression loss (e.g., MAE), and \( \mathcal{L}_{\text{contrast}} \) is the contrastive re-weighted loss as described above. The hyperparameter \( \alpha \in [0,1] \) controls the trade-off between the two objectives.

\section{Experimental evaluation}
\label{sec:exp}
\textbf{Datasets. }
We utilize 12 real-world temporal graphs with varying scales. 
Table~\ref{tab:data} in Appendix~\ref{appendix:datasets} provides their details. 
All
data sets are publicly available from  online  repositories, including netzschleuder~\cite{netzschleuder}, SNAP~\cite{snap}, SocioPatterns~\cite{fournet2014contact}, Konect~\cite{kunegis2013konect} and Network Repository~\cite{nr-temporal}.

\textbf{Compared Baselines. } 
To evaluate the computational efficiency, we compare CLGNN with the exact algorithm ETBC~\cite{TBC24}. 
To  evaluate effectiveness, we compare CLGNN against several baselines:  (i) \textbf{Static GNNs}, including GCN~\cite{GCN}; 
DrBC~\cite{FanZDCSL19}, a GNN-based classifier targeting high BC node identification; and GNN-Bet~\cite{MauryaLM19}, a supervised model that predicts BC scores from node embeddings.  (ii) \textbf{Temporal GNNs}, including TGAT~\cite{xu2020inductive}, which models time-evolving interactions using attention and continuous time encoding; TATKC~\cite{ZhangF00F24}, which uses time-injected attention for centrality ranking;  and DBGNN~\cite{HeegS24}, the SOTA temporal GNN framework built upon De Bruijn graphs for centrality prediction.  
(iii) \textbf{Contrastive learning models},  IIOIwD~\cite{IIOIwD}, which estimates node importance using contrastive sampling. To further validate our contrastive learning module, we compare CLGNN with its variants using over- and under-sampling, Tweedie loss, graph augmentation (node drop, edge perturbation, feature masking), and also without the module. 

\textbf{Evaluation Metrics.} We use three metrics to assess performance: mean absolute error (MAE) for prediction accuracy,  Spearman correlation for ranking quality, and HitsIn@$k$ scores (HitsIn10, 30, 50) to measure how well top predictions match the ground truth. 


Our training dataset consists of 50 real-world temporal networks from Network Repository~\cite{nr-temporal} and Konect~\cite{kunegis2013konect}. For every dataset, we repeated each experiment 30 times, and we reported the mean and the standard deviation of all scores.  The associated learning rates
as well as all other hyperparameters of the models are reported in Appendix~\ref{appendix:datasets}. 

\noindent\textbf{Next, we seek to answer the following research questions:}

\textbf{RQ1.} \quad How does CLGNN perform compared to static GNNs (GCN, DrBC, GNN-Bet)?

\textbf{RQ2.} \quad How does CLGNN compare to state-of-the-art temporal GNNs (TGAT, TATKC, DBGNN), and what are the benefits of its temporal path-count and contrastive learning designs?

\textbf{RQ3.} \quad How much faster is CLGNN than the leading exact method ETBC? 

\textbf{RQ4.} \quad How robust does CLGNN’s contrastive learning handle strong TBC imbalance, and how does it compare to sampling-based remedies (e.g., IIOIwD, random augmentation, Tweedie loss, over- and under-sampling)?

\textbf{Discussion of results. }
As shown in Table~\ref{tab:tbc_comparison}, considering \textbf{RQ1},  CLGNN consistently outperforms static GNN baselines across all datasets and metrics. It achieves 2.1~$\times$ to 31.4~$\times$ lower MAE and up to 16.5~$\times$ higher Spearman correlation. The gains are especially prominent on large-scale graphs with high temporal variation (e.g., sx-mathoverflow, wikiedits-se, ia-reality-call). Additional HitsIn@$k$ results are provided in Appendix~\ref{appendix:Results}, Table~\ref{tab:hit_comparison_std}. These results indicate that the importance of modeling temporal dependencies for accurate TBC estimation. 


\begin{table*}[htbp]
\centering
\caption{Performance comparison with static GNN baselines}
\label{tab:tbc_comparison}
\resizebox{\textwidth}{!}{
\begin{tabular}{l|cc|cc|cc|cc}
\toprule
\textbf{Dataset} & \multicolumn{2}{c|}{\textbf{CLGNN}} & \multicolumn{2}{c|}{GCN} & \multicolumn{2}{c|}{DrBC} & \multicolumn{2}{c}{GNN-Bet} \\
\cmidrule(r){2-3} \cmidrule(r){4-5} \cmidrule(r){6-7} \cmidrule(r){8-9}
& MAE & Spearman & MAE & Spearman & MAE & Spearman & MAE & Spearman \\
\midrule
sp\_hospital & \textbf{13.62 $\pm$ 0.33} & \textbf{0.74 $\pm$ 0.02} & 30.93 $\pm$ 0.29 & -0.19 $\pm$ 0.03 & 28.86 $\pm$ 0.32 & 0.53 $\pm$ 0.02 & 31.35 $\pm$ 0.26 & -0.55 $\pm$ 0.03 \\
sp\_hypertext & \textbf{40.62 $\pm$ 0.80} & \textbf{0.49 $\pm$ 0.02} & 67.14 $\pm$ 0.56 & 0.12 $\pm$ 0.01 & 71.65 $\pm$ 1.15 & 0.10 $\pm$ 0.02 & 69.48 $\pm$ 0.65 & -0.04 $\pm$ 0.02 \\
sp\_workplace & \textbf{12.33 $\pm$ 0.21} & \textbf{0.45 $\pm$ 0.02} & 31.73 $\pm$ 0.42 & -0.18 $\pm$ 0.02 & 36.68 $\pm$ 0.28 & 0.04 $\pm$ 0.02 & 36.40 $\pm$ 0.53 & -0.35 $\pm$ 0.01 \\
Highschool2011 & \textbf{26.69 $\pm$ 0.22} & \textbf{0.52 $\pm$ 0.03} & 111.67 $\pm$ 1.76 & 0.21 $\pm$ 0.02 & 84.48 $\pm$ 0.64 & 0.46 $\pm$ 0.02 & 75.31 $\pm$ 0.48 & -0.39 $\pm$ 0.02 \\
Highschool2012 & \textbf{157.54 $\pm$ 1.28} & \textbf{0.51 $\pm$ 0.02} & 423.73 $\pm$ 4.81 & 0.41 $\pm$ 0.03 & 307.25 $\pm$ 3.99 & 0.41 $\pm$ 0.03 & 193.89 $\pm$ 2.23 & -0.21 $\pm$ 0.01 \\
Highschool2013 & \textbf{160.59 $\pm$ 0.99} & \textbf{0.53 $\pm$ 0.02} & 634.24 $\pm$ 12.50 & 0.34 $\pm$ 0.02 & 408.42 $\pm$ 5.30 & 0.39 $\pm$ 0.03 & 273.95 $\pm$ 3.14 & -0.30 $\pm$ 0.03 \\
haggle & \textbf{5.22 $\pm$ 0.12} & \textbf{0.99 $\pm$ 0.03} & 58.32 $\pm$ 0.89 & 0.06 $\pm$ 0.02 & 50.17 $\pm$ 0.73 & 0.44 $\pm$ 0.02 & 92.58 $\pm$ 1.15 & 0.06 $\pm$ 0.03 \\
ia-reality-call & \textbf{170.19 $\pm$ 2.90} & \textbf{0.47 $\pm$ 0.01} & 663.26 $\pm$ 8.10 & 0.12 $\pm$ 0.03 & 535.28 $\pm$ 9.50 & 0.14 $\pm$ 0.03 & 425.77 $\pm$ 4.85 & -0.24 $\pm$ 0.03 \\
infectious & \textbf{106.93 $\pm$ 2.05} & \textbf{0.55 $\pm$ 0.02} & 451.61 $\pm$ 4.91 & 0.01 $\pm$ 0.02 & 376.90 $\pm$ 5.98 & 0.30 $\pm$ 0.02 & 344.98 $\pm$ 3.34 & -0.10 $\pm$ 0.01 \\
wikiedits-se & \textbf{16.54 $\pm$ 0.09} & \textbf{0.51 $\pm$ 0.02} & 399.47 $\pm$ 4.05 & 0.09 $\pm$ 0.03 & 211.16 $\pm$ 2.99 & 0.28 $\pm$ 0.03 & 315.38 $\pm$ 3.59 & -0.22 $\pm$ 0.02 \\
sx-mathoverflow & \textbf{264.44 $\pm$ 6.45} & \textbf{0.21 $\pm$ 0.02} & 1073.67 $\pm$ 17.11 & 0.01 $\pm$ 0.01 & 983.01 $\pm$ 15.57 & 0.014 $\pm$ 0.027 & 973.24 $\pm$ 17.91 & -0.41 $\pm$ 0.02 \\
superuser & \textbf{557.28 $\pm$ 12.06} & \textbf{0.41 $\pm$ 0.03} & 17524.43 $\pm$ 202.91 & 0.04 $\pm$ 0.03 & 14690.53 $\pm$ 144.62 & 0.12 $\pm$ 0.03 & 9654.33 $\pm$ 185.00 & -0.39 $\pm$ 0.01 \\
\bottomrule
\end{tabular}
}
\vspace{-10pt}
\end{table*}

As shown in Table~\ref{tab:temporal_tbc}, regarding \textbf{RQ2},  CLGNN  outperforms temporal GNN baselines across all datasets in MAE and reaches the highest Spearman correlation on 7 out of 12 datasets compared to DBGNN. 
Compared to TGAT and TATKC, CLGNN reduces MAE by 1.1~$\times$ to 9.4~$\times$ , with the large gains on highly dynamic datasets like sp\_hospital, haggle, sx-mathoverflow. Against DBGNN, which leverages high-order De Bruijn path modeling, CLGNN performs better across graphs of different sizes and densities, with up to 5.7~$\times$ lower MAE and 1.1 to 3.9~$\times$ higher rank correlation on datasets like Highschool2011 and ia-reality-call.  It also maintains stable ranking performance on large-scale datasets like superuser.  HitsIn@$k$ results are provided in Appendix~\ref{appendix:Results}, Table~\ref{tab:temporal_hits}. Additionally, as illustrated in Appendix~\ref{appendix:Results_paths} Table~\ref{tab:temporal_path_settings}, CLGNN generalizes well under various optimal paths. 
 Overall, these results show that the designs of temporal path count encoding and contrastive supervision in CLGNN enable more accurate and robust performance on temporal graphs.  
To address \textbf{RQ3}, we compare the efficiency of CLGNN with the exact TBC algorithm. As shown in  Appendix~\ref{appendix:Results} Figure~\ref{fig:clgnn_etbc_runtimes}, CLGNN achieves substantial speedup, ranging from 4.8$\times$ to 663.7$\times$, across all datasets. 


\vspace{-10pt}
\begin{table*}[htbp]
\centering
\caption{Performance comparison with temporal GNN baselines}
\label{tab:temporal_tbc}
\resizebox{\textwidth}{!}{
\begin{tabular}{l|cc|cc|cc|cc}
\toprule
\textbf{Dataset} & \multicolumn{2}{c|}{\textbf{CLGNN}} & \multicolumn{2}{c|}{\textbf{TGAT}} & \multicolumn{2}{c|}{\textbf{TATKC}} & \multicolumn{2}{c}{\textbf{DBGNN}} \\
\cmidrule(r){2-3} \cmidrule(r){4-5} \cmidrule(r){6-7} \cmidrule(r){8-9}
& MAE & Spearman & MAE & Spearman & MAE & Spearman & MAE & Spearman \\
\midrule
sp\_hospital     & \textbf{13.62 $\pm$ 0.33} & 0.74 $\pm$ 0.02 & 31.07 $\pm$ 0.38 & -0.15 $\pm$ 0.01 & 31.87 $\pm$ 0.45 & -0.05 $\pm$ 0.02 & 24.25 $\pm$ 0.23 & \textbf{0.83 $\pm$ 0.03} \\
sp\_hypertext    & \textbf{40.62 $\pm$ 0.80} & 0.49 $\pm$ 0.02 & 66.13 $\pm$ 0.75 & 0.39 $\pm$ 0.02 & 71.07 $\pm$ 0.71 & 0.27 $\pm$ 0.03 & 66.58 $\pm$ 0.76 & \textbf{0.84 $\pm$ 0.02} \\
sp\_workplace    & \textbf{12.33 $\pm$ 0.21} & 0.45 $\pm$ 0.02 & 35.52 $\pm$ 0.26 & -0.50 $\pm$ 0.02 & 36.98 $\pm$ 0.50 & -0.19 $\pm$ 0.03 & 70.80 $\pm$ 0.59 & \textbf{0.59 $\pm$ 0.03} \\
Highschool2011   & \textbf{26.69 $\pm$ 0.22} & \textbf{0.52 $\pm$ 0.03} & 74.70 $\pm$ 0.47 & 0.17 $\pm$ 0.03 & 75.75 $\pm$ 1.00 & 0.28 $\pm$ 0.03 & 34.56 $\pm$ 0.38 & 0.46 $\pm$ 0.02 \\
Highschool2012   & \textbf{157.54 $\pm$ 1.28} & 0.51 $\pm$ 0.02 & 193.43 $\pm$ 1.60 & 0.23 $\pm$ 0.01 & 194.97 $\pm$ 1.39 & 0.23 $\pm$ 0.03 & 203.34 $\pm$ 1.63 & \textbf{0.54 $\pm$ 0.03} \\
Highschool2013   & \textbf{160.59 $\pm$ 0.99} & 0.53 $\pm$ 0.02 & 272.99 $\pm$ 2.50 & 0.04 $\pm$ 0.02 & 274.66 $\pm$ 1.90 & 0.09 $\pm$ 0.02 & 496.09 $\pm$ 5.32 & \textbf{0.66 $\pm$ 0.01} \\
haggle           & \textbf{5.22 $\pm$ 0.12} & \textbf{0.99 $\pm$ 0.03} & 48.85 $\pm$ 0.31 & 0.42 $\pm$ 0.03 & 48.71 $\pm$ 0.31 & 0.47 $\pm$ 0.02 & 11.28 $\pm$ 0.08 & 0.63 $\pm$ 0.02 \\
ia-reality-call  & \textbf{170.19 $\pm$ 2.90} & \textbf{0.47 $\pm$ 0.01} & 316.87 $\pm$ 3.46 & 0.18 $\pm$ 0.02 & 246.85 $\pm$ 3.21 & 0.11 $\pm$ 0.03 & 190.66 $\pm$ 2.46 & 0.12 $\pm$ 0.02 \\
infectious       & \textbf{106.93 $\pm$ 2.05} & \textbf{0.55 $\pm$ 0.02} & 195.39 $\pm$ 1.70 & 0.22 $\pm$ 0.02 & 203.15 $\pm$ 2.58 & 0.41 $\pm$ 0.01 & 114.32 $\pm$ 1.21 & 0.26 $\pm$ 0.03 \\
wikiedits-se     & \textbf{16.54 $\pm$ 0.09} & \textbf{0.51 $\pm$ 0.02} & 138.73 $\pm$ 1.52 & 0.19 $\pm$ 0.03 & 25.43 $\pm$ 0.18 & 0.23 $\pm$ 0.03 & 31.50 $\pm$ 0.28 & 0.26 $\pm$ 0.03 \\
sx-mathoverflow  & \textbf{264.44 $\pm$ 6.45} & \textbf{0.21 $\pm$ 0.02} & 308.75 $\pm$ 3.97 & 0.09 $\pm$ 0.01 & 292.76 $\pm$ 3.69 & 0.17 $\pm$ 0.02 & 278.78 $\pm$ 1.75 & 0.06 $\pm$ 0.03 \\
superuser        & \textbf{557.28 $\pm$ 12.06} & \textbf{0.41 $\pm$ 0.03} & 728.01 $\pm$ 9.80 & 0.16 $\pm$ 0.02 & 606.73 $\pm$ 7.53 & 0.20 $\pm$ 0.03 & 570.70 $\pm$ 7.93 & 0.34 $\pm$ 0.02 \\
\bottomrule
\end{tabular}
}
\vspace{-10pt}
\end{table*}

Considering \textbf{RQ4}, Table~\ref{tab:ablation_contrastive} presents an ablation study of CLGNN's contrastive learning module. Compared to IIOIwD, CLGNN achieves much higher ranking stability. For instance, on infectious, Spearman correlation improves from -0.49 to 0.55, a relative gain of over 200\%. Removing the contrastive module ("No Contrastive") also leads to noticeable degradation. On highschool2012, Spearman drops from 0.51 to -0.26, and MAE increases by 1.2$\times$. 
Over- and under- sampling methods perform poorly on large-scale graphs. On highschool2011 and highschool2013, MAE 
is more than 13$\times$ higher than CLGNN's. 
Tweedie loss 
fails to preserve rank order under strong imbalance. 
Random augmentation techniques offer limited and inconsistent improvements. On highschool2011, Spearman is -0.04, compared to 0.52 with CLGNN. 
 HitsIn@$k$ results are provided in Appendix~\ref{appendix:Results}, Table~\ref{tab:contrastive_hits_part1} and Table~\ref{tab:contrastive_hits_part2}. These results confirm that CLGNN’s contrastive module is more robust and effective than alternative imbalance-handling methods, particularly in preserving ranking consistency and regression accuracy under skewed TBC distributions.

We also perform hyperparameter tuning, with results shown in Appendix~\ref{appendix:Results_Hyper}, where we analyze model performance across different  $(\alpha, \lambda)$ settings (see Figure~\ref{fig:hyperparam_all}). Additionally, we assess prediction accuracy across discretized TBC value ranges (zero, mid, and high) in Appendix~\ref{appendix:Results}, Table~\ref{tab:discretized_mae_ranges}. 
\vspace{-10pt}
\begin{table*}[htbp]
\centering
\caption{Ablation study of contrastive learning variants}
\label{tab:ablation_contrastive}
\resizebox{\textwidth}{!}{
\begin{tabular}{l|cc|cc|cc|cc|cc}
\toprule
\textbf{Dataset} 
& \multicolumn{2}{c|}{IIOIwD} 
& \multicolumn{2}{c|}{No Contrastive} 
& \multicolumn{2}{c|}{Over- and Under- sampling} 
& \multicolumn{2}{c|}{Tweedie Loss (Weighted)} 
& \multicolumn{2}{c}{Graph Aug (Random)} \\
\cmidrule(r){2-3} \cmidrule(r){4-5} \cmidrule(r){6-7} \cmidrule(r){8-9} \cmidrule(r){10-11}
& MAE & Spearman & MAE & Spearman & MAE & Spearman & MAE & Spearman & MAE & Spearman \\
\midrule
sp\_hospital     & 31.05 $\pm$ 0.19 & -0.29 $\pm$ 0.03 & 30.82 $\pm$ 0.51 & 0.28 $\pm$ 0.03 & 149.01 $\pm$ 2.78 & -0.05 $\pm$ 0.02 & 30.85 $\pm$ 0.46 & -0.07 $\pm$ 0.03 & 34.93 $\pm$ 0.50 & 0.22 $\pm$ 0.03 \\
sp\_hypertext    & 70.14 $\pm$ 1.38 & -0.21 $\pm$ 0.01 & 70.63 $\pm$ 0.86 & -0.37 $\pm$ 0.03 & 431.48 $\pm$ 2.44 & 0.19 $\pm$ 0.03 & 71.31 $\pm$ 0.51 & -0.58 $\pm$ 0.03 & 131.20 $\pm$ 0.93 & 0.27 $\pm$ 0.01 \\
sp\_workplace    & 35.82 $\pm$ 0.31 & -0.04 $\pm$ 0.01 & 36.38 $\pm$ 0.24 & -0.12 $\pm$ 0.03 & 65.67 $\pm$ 0.70 & -0.45 $\pm$ 0.02 & 32.40 $\pm$ 0.22 & 0.13 $\pm$ 0.01 & 29.00 $\pm$ 0.58 & -0.16 $\pm$ 0.02 \\
highschool2011   & 74.84 $\pm$ 1.22 & -0.38 $\pm$ 0.03 & 75.27 $\pm$ 0.69 & -0.36 $\pm$ 0.01 & 505.77 $\pm$ 4.94 & 0.20 $\pm$ 0.01 & 69.49 $\pm$ 0.43 & -0.13 $\pm$ 0.02 & 89.18 $\pm$ 1.02 & -0.04 $\pm$ 0.03 \\
highschool2012   & 193.82 $\pm$ 2.50 & -0.43 $\pm$ 0.03 & 193.86 $\pm$ 3.28 & -0.26 $\pm$ 0.01 & 1428.01 $\pm$ 14.72 & 0.12 $\pm$ 0.03 & 184.34 $\pm$ 1.98 & -0.41 $\pm$ 0.01 & 181.98 $\pm$ 3.43 & -0.15 $\pm$ 0.01 \\
highschool2013   & 273.49 $\pm$ 5.09 & 0.02 $\pm$ 0.02 & 270.93 $\pm$ 3.44 & 0.20 $\pm$ 0.03 & 2186.28 $\pm$ 41.54 & 0.10 $\pm$ 0.01 & 222.45 $\pm$ 1.60 & 0.06 $\pm$ 0.02 & 230.23 $\pm$ 1.95 & -0.04 $\pm$ 0.03 \\
haggle           & 48.19 $\pm$ 0.88 & -0.57 $\pm$ 0.02 & 52.25 $\pm$ 0.62 & 0.49 $\pm$ 0.01 & 1670.79 $\pm$ 18.09 & 0.58 $\pm$ 0.02 & 169.75 $\pm$ 1.99 & -0.03 $\pm$ 0.03 & 258.64 $\pm$ 3.30 & 0.54 $\pm$ 0.03 \\
ia-reality-call  & 297.40 $\pm$ 1.89 & -0.48 $\pm$ 0.03 & 209.38 $\pm$ 1.85 & 0.33 $\pm$ 0.02 & 1022.29 $\pm$ 18.85 & 0.28 $\pm$ 0.03 & 234.29 $\pm$ 2.42 & 0.23 $\pm$ 0.02 & 293.25 $\pm$ 4.73 & 0.29 $\pm$ 0.01 \\
infectious       & 124.77 $\pm$ 1.59 & -0.49 $\pm$ 0.03 & 184.36 $\pm$ 1.24 & 0.43 $\pm$ 0.02 & 1668.81 $\pm$ 11.99 & 0.33 $\pm$ 0.02 & 194.25 $\pm$ 2.23 & 0.08 $\pm$ 0.02 & 212.49 $\pm$ 1.34 & -0.03 $\pm$ 0.03 \\
wikiedits-se     & 265.74 $\pm$ 2.70 & -0.25 $\pm$ 0.02 & 85.42 $\pm$ 1.48 & 0.29 $\pm$ 0.03 & 584.82 $\pm$ 7.23 & 0.35 $\pm$ 0.03 & 263.23 $\pm$ 4.48 & 0.08 $\pm$ 0.01 & 327.70 $\pm$ 2.64 & 0.17 $\pm$ 0.03 \\
sx-mathoverflow  & 292.22 $\pm$ 2.39 & -0.51 $\pm$ 0.03 & 394.73 $\pm$ 3.85 & 0.11 $\pm$ 0.01 & 521.96 $\pm$ 9.84 & 0.41 $\pm$ 0.02 & 372.92 $\pm$ 3.85 & 0.13 $\pm$ 0.02 & 389.11 $\pm$ 5.19 & 0.09 $\pm$ 0.02 \\
superuser        & 1234.66 $\pm$ 12.85 & -0.54 $\pm$ 0.03 & 691.24 $\pm$ 5.04 & 0.16 $\pm$ 0.03 & 1660.90 $\pm$ 28.27 & 0.54 $\pm$ 0.03 & 746.62 $\pm$ 4.27 & 0.20 $\pm$ 0.02 & 927.87 $\pm$ 4.74 & 0.02 $\pm$ 0.03 \\
\bottomrule
\end{tabular}
}
\vspace{-15pt}
\end{table*}

\section{Conclusion}
\label{sec:con}
This paper proposes CLGNN, a contrastive learning-based graph neural network designed to accurately predict TBC. CLGNN incorporates a dual aggregation mechanism that combines temporal path count encoding and time encoding to learn expressive node representations.  It also includes a stability-based clustering-guided contrastive learning module to better distinguish nodes across different TBC value ranges. Extensive experiments on real-world datesets with varying sizes demonstrate the effectiveness of CLGNN. 

We believe our work contributes meaningfully to the ongoing advancement of temporal graph learning. However, several open questions remain for future investigation. Due to computational resource limitations, we did not perform an extensive hyperparameter search. In particular, sensitivity analyses for parameters such as the time encoding basis dimension, the number of clusters in the KContrastNet module, and the attention head size in the dual aggregation mechanism were not conducted. 
While we reported results using reasonable default settings and tuned the learning rate via grid search, a more comprehensive exploration of these hyperparameters may yield further improvements. 
For future work, we plan to explore advanced temporal graph augmentation techniques to further improve model generalization and robustness. In addition, we aim to extend our approach to predict other types of temporal centrality measures to evaluate the broader applicability of CLGNN in temporal graph analysis.

\section*{Acknowledgement} This work was supported by the National Natural Science Foundation of China under Grant No. 62302451 and No. 62276233.  Bin Cao and Jing Fan are the corresponding authors of the work.

\small 
\begingroup
\bibliographystyle{abbrv}
\bibliography{refer}
\endgroup

\newpage
\appendix
\section{Details of the CLGNN——Representation Module}\label{appendix:CLGNN}
As shown in Figure~\ref{fig:rep_architecture}, the representation module integrates structural and temporal information via dual aggregation mechanisms, incorporating both path count encoding and continuous time encoding. This module is responsible for generating discriminative node representations that reflect the temporal dynamics of the input graph.
The model first constructs an instance graph and computes a list of reachable temporal path counts for each node pair. It then encodes both continuous time intervals and path counts to reflect the dynamic and structural context of each interaction. The message function incorporates these encodings and supports two aggregation strategies: a mean-based node-level aggregator that stabilizes learning under temporal variance, and a multi-head attention mechanism that adaptively weights important temporal edges. Finally, the two representations are fused via a weighted combination to yield expressive and context-aware node embeddings.

\begin{figure}[htbp]  
  \centering  
  \includegraphics[width=\textwidth]{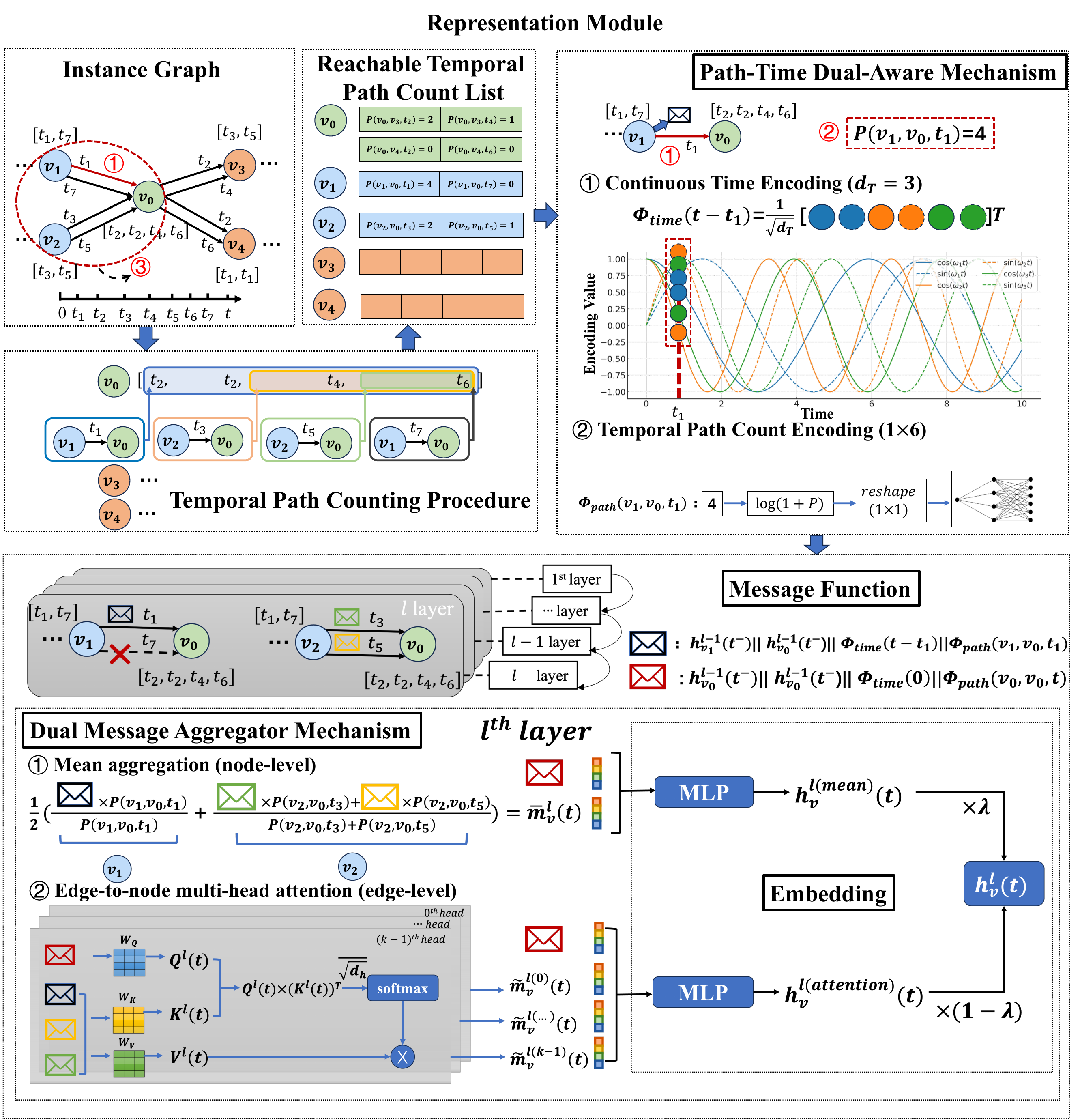}  
  \caption{Overview of the representation module}  
  \label{fig:rep_architecture}  
\vspace{-10pt}
\end{figure}

\newpage
\section{Details of the CLGNN——Prediction Module}\label{appendix:CLGNN2}
Figure~\ref{fig:pre_architecture} depicts the prediction module, which includes a regression head to estimate the temporal betweenness centrality (TBC) values, as well as a contrastive learning module designed to enhance representation separability under extreme imbalance. These components jointly optimize the model to achieve both accurate centrality estimation and robustness to data skew.
The upper part introduces \textbf{KContrastNet}, a stability-driven clustering framework that applies stratified sampling and repeated K-means clustering to identify a optimal number of clusters in the node embedding space. Based on the cluster-wise TBC statistics, informative positive and negative sample pairs are selected to form a contrastive learning objective. The lower part shows \textbf{ValueNet}, a multilayer perceptron that maps node embeddings to scalar TBC predictions. The final training objective integrates the regression loss and contrastive loss via a tunable parameter to balance accurate value estimation with discriminative representation learning.

\begin{figure}[htbp]  
  \centering  
  \includegraphics[width=\textwidth]{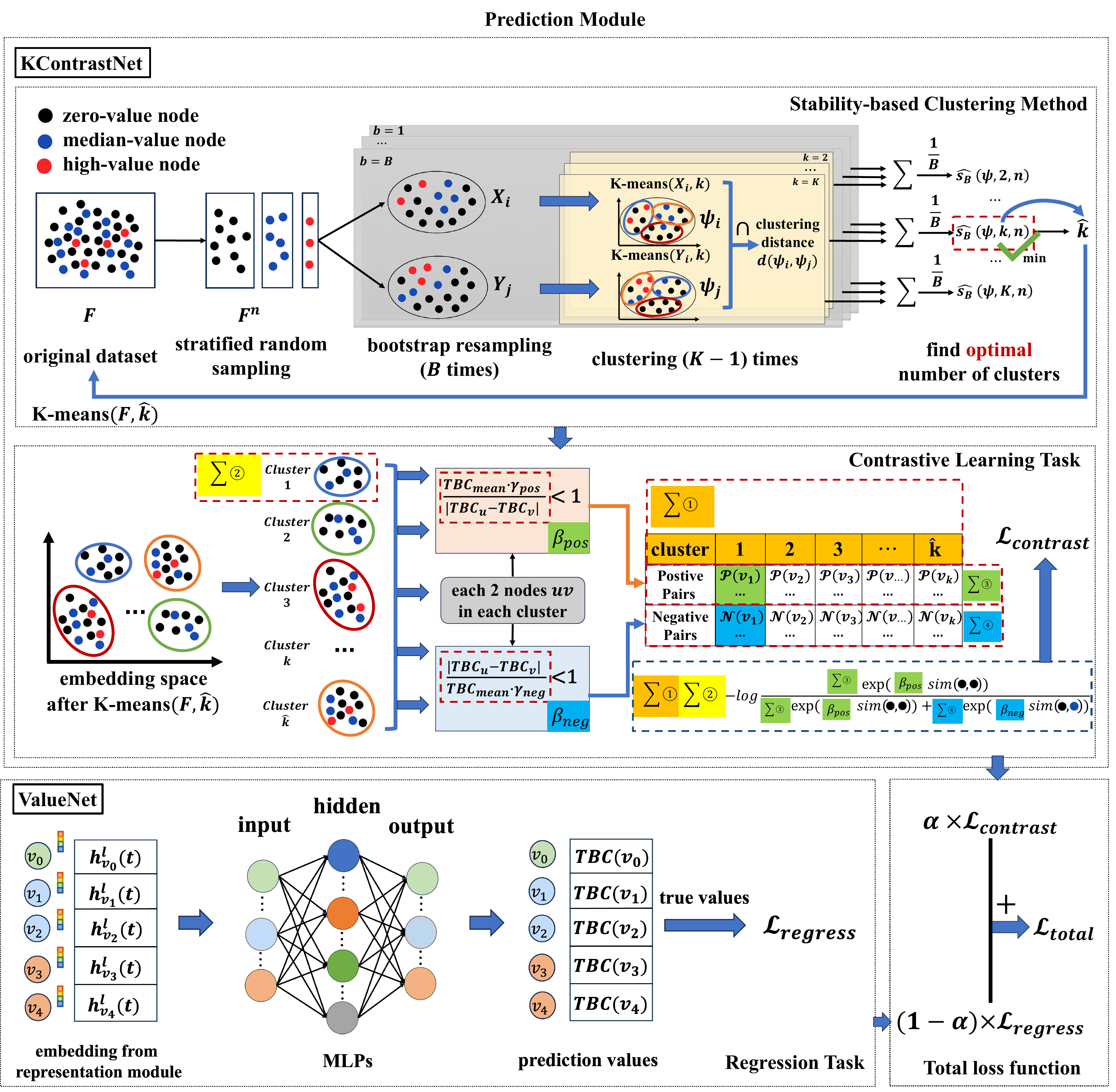}  
  \caption{Overview of the prediction module}  
  \label{fig:pre_architecture}  
\vspace{-10pt}
\end{figure}

\newpage
\section{Algorithm and Complexity Analysis for Stability-Based Cluster Selection}\label{appendix:algorithm}

\begin{algorithm}[H]
\caption{\small{~Stability-Based Estimation of Optimal Cluster Number \( \hat{k} \)}} \label{alg:optimal-k}
\small{
\textbf{Inputs:} \\
\quad Feature set $F = \{f_1, f_2, \dots, f_N\}$ \tcp*[r]{node-level representation set} 
\quad Sampling rate $r \in (0,1)$ \tcp*[r]{fraction of data to retain} 
\quad Number of bootstrap pairs $B$, Candidate cluster numbers $k \in \{2, \dots, K\}$ 

\textbf{Output:} Optimal cluster number $\hat{k}$ \\

\vspace{0.2em}
\textbf{Step 1: Sampling and Bootstrap} \\
$F^n \leftarrow \text{StratifiedSample}(F,~r)$ \tcp*[r]{retain TBC distribution via stratified sampling} 
\For{$i=1,\dots,B$}{
    $(X_i, Y_i) \leftarrow \text{BootstrapResamplePair}(F^n)$ \tcp*[r]{sample two i.i.d. subsets}
}

\vspace{0.2em}
\textbf{Step 2: Compute clustering instability for each } $k$

\For{$k=2,\dots,K$}{
    $s_k \leftarrow 0$\;
    \For{$i=1,\dots,B$}{
        $\psi_{X_i} \leftarrow \text{KMeans}(X_i,~k)$ \tcp*[r]{cluster $X_i$ into $k$ groups}
        
        $\psi_{Y_i} \leftarrow \text{KMeans}(Y_i,~k)$ \tcp*[r]{cluster $Y_i$ independently} 
        
        $d_i \leftarrow d(\psi_{X_i}, \psi_{Y_i})$ \tcp*[r]{compute clustering distance per Def.~\ref{defn:CD}} 
        
        $s_k \leftarrow s_k + d_i$ \tcp*[r]{accumulate distance across $B$ trials}
    }
    $\hat{s}_B(k) \leftarrow \frac{1}{B} \cdot s_k$ \tcp*[r]{average instability for cluster number $k$}
    
}

\vspace{0.2em}
\textbf{Step 3: Select optimal } $k$ 

$\hat{k} \leftarrow \arg\min_{k \in \{2,\dots,K\}} \hat{s}_B(k)$ \tcp*[r]{minimum instability $\Rightarrow$ best $k$}

\textbf{Return:} $\hat{k}$\;
}
\end{algorithm}

\subsection*{Complexity Analysis of Algorithm~\ref{alg:optimal-k}}

We analyze the time complexity of Algorithm~\ref{alg:optimal-k}, which estimates the optimal number of clusters \( \hat{k} \) using a stability-based bootstrap method.

\paragraph{Notation.}
Let \( N \) denote the total number of feature samples in the input set \( F \), and let \( r \in (0,1) \) be the stratified sampling ratio. Let \( n = rN \) be the size of the reduced sample set \( F^n \). Let \( B \) be the number of bootstrap resample pairs, and \( K \) the maximum number of candidate clusters.

\paragraph{Step 1: Sampling and Bootstrapping.}
\begin{itemize}[topsep=0pt, itemsep=1pt, parsep=0pt, partopsep=0pt]
    \item Stratified sampling requires \( \mathcal{O}(N) \) time.
    \item Constructing \( B \) independent resample pairs of size \( n \) requires \( \mathcal{O}(Bn) \).
\end{itemize}

\paragraph{Step 2: Clustering Instability Computation.}
For each candidate cluster number \( k \in \{2, \dots, K\} \), and each bootstrap pair \( (X_i, Y_i) \), the following operations are performed:
\begin{itemize}[topsep=0pt, itemsep=1pt, parsep=0pt, partopsep=0pt]
    \item \textbf{KMeans clustering:} Two independent KMeans clusterings of size \( n \) and \( k \) clusters take \( \mathcal{O}(nkT) \) time each, where \( T \) is the number of iterations (assumed constant in practice).
    \item \textbf{Clustering distance computation:} For each pair of clusterings, computing the pairwise agreement distance (as defined in Definition~\ref{defn:CD}) over \( \mathcal{O}(n^2) \) pairs requires \( \mathcal{O}(n^2) \) time.
\end{itemize}

Therefore, the total time complexity for instability estimation is:
\[
\begin{aligned}
\mathcal{O}\left( \sum_{k=2}^{K} \sum_{i=1}^{B} (2nkT + n^2) \right) 
&= \mathcal{O}\left( B \cdot \sum_{k=2}^{K} (2nkT + n^2) \right) \\
&= \mathcal{O}\left( BnTK^2 + Bn^2K \right)
\end{aligned}
\]

\paragraph{Total Time Complexity.}
Combining both steps, the overall time complexity is:
\[
\boxed{
\mathcal{O}(N) + \mathcal{O}(Bn) + \mathcal{O}(BnTK^2 + Bn^2K) = \mathcal{O}(BnTK^2 + Bn^2K)
}
\]

\paragraph{Remarks.}
This complexity is dominated by the clustering instability computations in Step 2. In practice, setting \( B \ll N \), \( r \ll 1 \), and bounding \( K \) (e.g., \( K < 10 \)) ensures tractability. Further efficiency can be achieved by approximating the clustering distance or subsampling intra-pair comparisons.

\section{Justification of i.i.d. Bootstrap Sample Pairs}\label{appendix:iid}

To ensure that each bootstrap sample pair \( (X_i, Y_i) \) used in the stability-based clustering method is independent and identically distributed (i.i.d.), we provide the following formal statistical justification grounded in classical nonparametric bootstrap theory.

\paragraph{Preliminaries.}
Let \( F^n = \{x_1, x_2, \dots, x_n\} \) be a stratified subsample drawn from the full dataset \( F \), such that the empirical distribution \( \hat{F}_n \) preserves the marginal structure of the original data (e.g., TBC value distribution). Define the sample space \( \mathcal{X} \subset \mathbb{R}^d \) where each \( x_i \in \mathcal{X} \) is a feature vector.

\paragraph{Bootstrap Sampling.}
We define a single bootstrap sample \( X_i = \{x^{(i)}_1, \dots, x^{(i)}_n\} \) as a sequence of i.i.d. random draws with replacement from \( F^n \), i.e.,
\[
x^{(i)}_j \overset{\text{i.i.d.}}{\sim} \hat{F}_n, \quad \forall j \in \{1, \dots, n\}.
\]
Therefore, each bootstrap sample \( X_i \sim \hat{F}_n^n \), the \( n \)-fold product distribution over \( \mathcal{X}^n \). The same applies to \( Y_i \), which is sampled independently using a separate draw.

\paragraph{Independence of Pairs.}
Since \( X_i \) and \( Y_i \) are constructed through independent sampling procedures, the joint distribution of the pair factorizes:
\[
P(X_i, Y_i) = P(X_i) \cdot P(Y_i) = \hat{F}_n^n \times \hat{F}_n^n.
\]
Hence, each pair \( (X_i, Y_i) \) is drawn independently from the same product distribution over \( \mathcal{X}^{2n} \), i.e.,
\[
(X_i, Y_i) \sim \hat{F}_n^n \times \hat{F}_n^n.
\]

\paragraph{Sequence of i.i.d. Pairs.}
The above procedure is repeated \( B \) times. Because each resampling of \( (X_i, Y_i) \) is constructed independently and from the same base distribution \( \hat{F}_n^n \times \hat{F}_n^n \), the collection \( \{(X_i, Y_i)\}_{i=1}^B \) forms an i.i.d. sequence:
\[
(X_1, Y_1), \dots, (X_B, Y_B) \overset{\text{i.i.d.}}{\sim} \hat{F}_n^n \times \hat{F}_n^n.
\]

\paragraph{Unbiasedness and Consistency of Instability Estimator.}
Let the clustering distance function be defined as \( d: \mathcal{X}^n \times \mathcal{X}^n \rightarrow \mathbb{R}_{\geq 0} \), such that for each pair \( (X_i, Y_i) \),
\[
Z_i := d(\psi_{X_i}, \psi_{Y_i}) \quad \text{is a real-valued random variable}.
\]
Because \( \{Z_i\}_{i=1}^B \) are i.i.d. realizations of a bounded statistic (e.g., \( d \in [0,1] \)), the empirical average
\[
\hat{s}_B(k) = \frac{1}{B} \sum_{i=1}^B Z_i
\]
which is an unbiased estimator of the expected instability \( \mathbb{E}[Z] \), and by the Weak Law of Large Numbers (WLLN), it is consistent:
\[
\hat{s}_B(k) \xrightarrow{P} \mathbb{E}[d(\psi_{X}, \psi_{Y})], \quad \text{as } B \rightarrow \infty.
\]

\paragraph{Conclusion.}
Thus, under the nonparametric bootstrap scheme, each sample pair \( (X_i, Y_i) \) is identically distributed from \( \hat{F}_n^n \times \hat{F}_n^n \), and independent across \( i \). This justifies that the sequence \( \{(X_i, Y_i)\}_{i=1}^B \) is i.i.d., which guarantees both the validity and convergence of the clustering instability estimator \( \hat{s}_B(k) \).

\paragraph{Connection to Estimator Properties.}
The i.i.d. structure of the bootstrap pairs \( (X_i, Y_i) \) established here forms the theoretical foundation for treating the clustering instability estimator \( \hat{s}_B(k) \) as a valid sample mean estimator. Its statistical properties are further analyzed in Appendix~\ref{appendix:instability_consistency}.

\section{Statistical Properties of the Clustering Instability Estimator}
\label{appendix:instability_consistency}

This analysis relies on the i.i.d. construction of bootstrap sample pairs, which is formally established in Appendix~\ref{appendix:iid}. Given this structure, \( \hat{s}_B(k) \) can be viewed as the empirical mean of i.i.d. evaluations of the clustering distance function.
To rigorously support the use of the clustering instability measure \( \hat{s}_B(k) \) as a criterion for selecting the optimal number of clusters, we analyze its statistical properties under the nonparametric bootstrap framework. In particular, we demonstrate that \( \hat{s}_B(k) \) is an unbiased and consistent estimator of the expected clustering disagreement under the empirical distribution \( \hat{F}_n \).

\paragraph{Problem Setup.}
Let \( F^n = \{x_1, \dots, x_n\} \) be a stratified subsample from the full dataset \( F \), and let \( \hat{F}_n \) denote its empirical distribution. For each bootstrap index \( i = 1, \dots, B \), the pair \( (X_i, Y_i) \) is drawn independently from \( \hat{F}_n^n \times \hat{F}_n^n \), where each \( X_i \) and \( Y_i \) is a size-\( n \) sample with replacement from \( F^n \).

Given a clustering algorithm \( \mathcal{C}(\cdot; k) \) such as K-means, and a distance function \( d: \mathcal{X}^n \times \mathcal{X}^n \to \mathbb{R}_{\geq 0} \) defined over clustering assignments (e.g., the expected pairwise label disagreement), define the random variable:

\[
Z_i := d(\psi_{X_i}, \psi_{Y_i}), \quad \text{where } \psi_{X_i} := \mathcal{C}(X_i; k),\ \psi_{Y_i} := \mathcal{C}(Y_i; k)
\]

Then, the clustering instability estimator is defined as the empirical mean:

\[
\hat{s}_B(k) := \frac{1}{B} \sum_{i=1}^{B} Z_i
\]

\paragraph{Unbiasedness under the Empirical Distribution.}
Since each pair \( (X_i, Y_i) \) is sampled i.i.d. from \( \hat{F}_n^n \times \hat{F}_n^n \), and assuming the clustering algorithm \( \mathcal{C} \) is a deterministic function of its input, it follows directly that:

\[
\mathbb{E}[\hat{s}_B(k)] = \mathbb{E}[Z_1] = \mathbb{E}_{(X,Y) \sim \hat{F}_n^n \times \hat{F}_n^n} \left[ d\big( \mathcal{C}(X; k), \mathcal{C}(Y; k) \big) \right]
\]

Thus, \( \hat{s}_B(k) \) is an unbiased estimator of the expected clustering disagreement under the empirical data distribution \( \hat{F}_n \).

\paragraph{Consistency via the Weak Law of Large Numbers.}
Since \( Z_i \in [0,1] \) by construction (e.g., using normalized pairwise disagreement), it is a bounded random variable with finite variance. By the Weak Law of Large Numbers (WLLN) for i.i.d. bounded sequences:

\[
\hat{s}_B(k) \xrightarrow{P} \mathbb{E}_{(X,Y) \sim \hat{F}_n^n \times \hat{F}_n^n} \left[ d\big( \mathcal{C}(X; k), \mathcal{C}(Y; k) \big) \right] \quad \text{as } B \to \infty
\]

That is, the clustering instability estimator converges in probability to its expected value under the empirical distribution.

\paragraph{Bootstrap Validity and Approximation to the True Distribution.}
While the above analysis holds under \( \hat{F}_n \), the ultimate goal is to approximate instability under the true data-generating distribution \( F \). Classical results in nonparametric bootstrap theory assert that the empirical measure \( \hat{F}_n \) converges weakly to \( F \) as \( n \to \infty \), and that functionals of resampled statistics converge accordingly under mild regularity conditions. Therefore, provided that the clustering distance function \( d(\cdot, \cdot) \) is measurable and the clustering algorithm is stable with respect to sampling perturbations, we have:

\[
\lim_{n \to \infty} \mathbb{E}_{(X,Y) \sim \hat{F}_n^n \times \hat{F}_n^n} \left[ d(\mathcal{C}(X), \mathcal{C}(Y)) \right] = \mathbb{E}_{(X,Y) \sim F^n \times F^n} \left[ d(\mathcal{C}(X), \mathcal{C}(Y)) \right]
\]

This implies that, asymptotically, \( \hat{s}_B(k) \) approximates the true expected instability under \( F \), justifying its use for model selection even in the absence of prior distributional knowledge.

\paragraph{Conclusion.}
The clustering instability estimator \( \hat{s}_B(k) \) is both an unbiased and consistent statistic under the empirical distribution \( \hat{F}_n \), and converges to the true instability under mild assumptions as \( n \to \infty \). This formalizes its validity as a statistically principled model selection criterion for determining the optimal number of clusters \( \hat{k} \).

\section{Theoretical Rationale for Intra-Cluster Contrastive Sampling}
\label{appendix:intracluster}

In our prediction framework, the selection of contrastive pairs is restricted to intra-cluster sampling based on a stability-driven K-means partition of the embedding space. While conventional contrastive frameworks often consider cross-cluster sampling to enhance diversity, we deliberately avoid this strategy for temporal betweenness centrality (TBC) prediction, owing to the unique structural and temporal sensitivities of the task. This section provides a principled justification for intra-cluster contrastive sampling from both a task-specific and geometric optimization perspective.

\subsection*{Task-Conditioned Semantic Consistency and Confounding Control}

Temporal betweenness centrality is a path-sensitive metric influenced not only by graph topology but also by the causal alignment of time-stamped edges. Clusters in the embedding space—learned through representations that encode path-time statistics—implicitly capture nodes with similar global spatiotemporal context (e.g., reachability profiles, path count distributions, or dynamic flow roles). Negative pairs sampled across clusters are likely to reside in distinct structural regimes, introducing latent confounders such as community boundaries, temporal reachability constraints, or phase-shifted diffusion fronts. These inter-cluster contrasts may reflect global distributional artifacts rather than TBC-specific discrepancies.

By contrast, constraining sampling to within-cluster pairs ensures semantic homogeneity at the structural-temporal level, enabling contrastive learning to focus on fine-grained differences in TBC values under shared global context. This mitigates the risk of shortcut learning and enhances task alignment of the learned embeddings.

\subsection*{Cross-Cluster Sampling Reduces Effective Coverage}

Let \( \hat{k} \) denote the number of clusters and \( V \) the set of all nodes. Let \( \psi_1, \psi_2, \dots, \psi_{\hat{k}} \) be the disjoint clusters identified by K-means, i.e., \( V = \bigcup_{i=1}^{\hat{k}} \psi_i \). The total number of possible cross-cluster node pairs is:

\[
|\mathcal{S}_{\text{cross}}| = \sum_{1 \leq i < j \leq \hat{k}} |\psi_i| \cdot |\psi_j|
\]

Since this quantity grows as \( O(|V|^2 / \hat{k}) \) under approximately balanced clustering, the total number appears large. However, the \emph{semantic validity} of these negative pairs (i.e., pairs that are structurally similar but differ in TBC) decays rapidly with cluster separation.

To formalize this, define the \textit{contrastive utility} of a negative pair \( (u, v) \) as:

\[
\mathcal{U}_{uv} := \mathbb{I}[ \text{sim}(h_u, h_v) > \epsilon ] \cdot \mathbb{I}[ |TBC(u) - TBC(v)| > \gamma \cdot TBC_{median} ]
\]

Here, \( \text{sim}(h_u, h_v) \) denotes the embedding similarity (e.g., cosine), and \( \mathbb{I}[\cdot] \) is the indicator function. For cross-cluster pairs:

\[
\mathbb{E}[\mathcal{U}_{uv} \mid u \in \psi_i, v \in \psi_j, i \neq j] \to 0
\]

as cluster separation increases and node pairs become trivially dissimilar. Thus, although \( |\mathcal{S}_{\text{cross}}| \) is large, its expected contribution to the InfoNCE gradient is negligible. Specifically, the contrastive loss gradient with respect to anchor \( u \) becomes:

\[
\nabla_{h_u} \mathcal{L}_{\text{InfoNCE}} \propto \sum_{v \in \mathcal{S}^-(u)} -\beta_{uv} \cdot \text{sim}(h_u, h_v)
\]

If \( \text{sim}(h_u, h_v) \approx 0 \), as is typical in cross-cluster sampling, the gradient vanishes—leading to \emph{representation collapse} and ineffective learning.

In contrast, intra-cluster sampling produces harder negatives—nodes that are structurally similar (high \( \text{sim}(h_u, h_v) \)) but semantically distinct in TBC—yielding stronger gradients and sharper decision boundaries in the learned space.

Let \( p_{\text{hard}}^{\text{cross}} \) and \( p_{\text{hard}}^{\text{intra}} \) denote the empirical proportion of informative negative pairs in cross- and intra-cluster sampling respectively. Then under high imbalance:

\[
p_{\text{hard}}^{\text{cross}} \ll p_{\text{hard}}^{\text{intra}}, \quad \text{especially as } \hat{k} \uparrow
\]

Thus, intra-cluster sampling provides higher effective coverage of meaningful negatives, with lower sample inefficiency and loss degeneracy.

\subsection*{Geometric Insights from InfoNCE and Alignment-Uniformity Tradeoff}

The InfoNCE loss~\cite{InfoNCE} aims to simultaneously maximize positive pair similarity and minimize negative pair similarity. However, as shown by~\cite{saunshi2019theoretical,wang2020understanding}, the loss is only informative if negative samples lie close to the anchor in representation space. Otherwise, gradients vanish, and learned representations collapse to trivial solutions.

Furthermore, Wang and Isola~\cite{wang2020understanding} establishes that effective contrastive learning is governed by a balance between:

\textbf{Alignment:} minimizing intra-class distances (for positive pairs);

\textbf{Uniformity:} repelling nearby but dissimilar samples across the manifold.

Intra-cluster sampling naturally preserves this balance: positive samples are tightly aligned (by design), while negative samples—though embedded closely—are semantically divergent in TBC, enhancing uniformity. Cross-cluster sampling violates both principles by creating easy positives and easy negatives.

\textbf{Conclusion.} Intra-cluster contrastive sampling introduces an inductive bias aligned with the path-time semantics of TBC. It avoids gradient vanishing, suppresses representational shortcuts, and encourages semantic disentanglement under structural coherence. Empirically and theoretically, it yields better sample efficiency, loss curvature, and prediction accuracy under extreme TBC imbalance.

\section{Theoretical Justification of Contrastive Loss Design}\label{appendix:contrastive_theory}

\subsection*{Monotonicity and Semantic Validity of Contrastive Weights}

We analyze the design of the pairwise contrastive weights:
\[
\beta_{uv} = \frac{TBC_{\text{median}} \cdot \gamma_{pos}}{|TBC(u) - TBC(v)|}, \qquad
\beta_{uw} = \frac{|TBC(u) - TBC(w)|}{TBC_{\text{median}} \cdot \gamma_{neg}}.
\]

\begin{lemma}[Monotonicity of Contrastive Weights]
Let \( \Delta_{uv} = |TBC(u) - TBC(v)| > 0 \), with \( \gamma_{pos}, \gamma_{neg} \in (0,1) \), and let \( TBC_{\text{median}} > 0 \). Then:
\begin{enumerate}
    \item The positive-pair weight \( \beta_{uv} \) is strictly decreasing with respect to \( \Delta_{uv} \);
    \item The negative-pair weight \( \beta_{uw} \) is strictly increasing with respect to \( \Delta_{uw} \).
\end{enumerate}
\end{lemma}

\begin{proof}
\begin{itemize}
    \item For \( \beta_{uv} = \dfrac{TBC_{\text{median}} \cdot \gamma_{pos}}{\Delta_{uv}} \), we have:
    \[
    \frac{d\beta_{uv}}{d\Delta_{uv}} = -\frac{TBC_{\text{median}} \cdot \gamma_{pos}}{\Delta_{uv}^2} < 0.
    \]
    Hence, \( \beta_{uv} \) decreases as the TBC difference increases, emphasizing the similarity of closely matched pairs.
    
    \item For \( \beta_{uw} = \dfrac{\Delta_{uw}}{TBC_{\text{median}} \cdot \gamma_{neg}} \), we have:
    \[
    \frac{d\beta_{uw}}{d\Delta_{uw}} = \frac{1}{TBC_{\text{median}} \cdot \gamma_{neg}} > 0.
    \]
    Therefore, \( \beta_{uw} \) increases linearly with TBC difference, encouraging separation of semantically distant pairs.
\end{itemize}
\end{proof}

This monotonic design ensures semantic alignment: more similar positive pairs receive larger weights, while more dissimilar negative pairs are pushed apart more aggressively.

\subsection*{Boundedness and Numerical Stability of Contrastive Loss}

We analyze the numerical behavior of the contrastive loss:
\[
\mathcal{L}_{\text{contrast}} = \sum_{i=1}^{\hat{k}} \sum_{u \in \psi_i} -\log \frac{
\sum\limits_{v \in \mathcal{S}^+(u)} \exp\left( \beta_{uv} \cdot \text{sim}(\mathbf{h}_u, \mathbf{h}_v)/\tau \right)}{
\sum\limits_{v \in \mathcal{S}^+(u)} \exp\left( \beta_{uv} \cdot \text{sim}(\mathbf{h}_u, \mathbf{h}_v)/\tau \right) + \sum\limits_{w \in \mathcal{S}^-(u)} \exp\left( \beta_{uw} \cdot \text{sim}(\mathbf{h}_u, \mathbf{h}_w)/\tau \right)
}.
\]

\begin{lemma}[Non-negativity and Upper Bound]
Assuming \( \text{sim}(\cdot,\cdot) \in [-1, 1] \), \( \beta_{uv}, \beta_{uw} > 0 \), and \( \tau > 0 \), the loss satisfies:
\[
\mathcal{L}_{\text{contrast}} \geq 0, \quad \text{and} \quad \mathcal{L}_{\text{contrast}} \leq \sum_{u} \log(|\mathcal{S}^-(u)| + |\mathcal{S}^+(u)|).
\]
\end{lemma}

\begin{proof}
Each term in the loss has the form:
\[
-\log \left( \frac{A}{A + B} \right) = \log \left( 1 + \frac{B}{A} \right),
\]
where \( A = \sum_{v \in \mathcal{S}^+(u)} \exp(\cdot) > 0 \), and \( B = \sum_{w \in \mathcal{S}^-(u)} \exp(\cdot) \geq 0 \). Since \( A > 0 \), the logarithm is well-defined, and:
\[
\log \left( 1 + \frac{B}{A} \right) \geq 0.
\]
For the upper bound, assuming all dot products are in \( [-1, 1] \), we have:
\[
\text{sim}(\mathbf{h}_u, \mathbf{h}_v) \leq 1, \quad \Rightarrow \quad \exp\left( \frac{\beta \cdot 1}{\tau} \right) \leq C,
\]
for some constant \( C \). Then the numerator and denominator are bounded by \( |\mathcal{S}^+(u)| \cdot C \) and \( (|\mathcal{S}^+(u)| + |\mathcal{S}^-(u)|) \cdot C \), hence the total loss is bounded above by \( \log(|\mathcal{S}^+| + |\mathcal{S}^-|) \).
\end{proof}

This guarantees that the loss is numerically stable, non-divergent, and bounded during optimization, regardless of sample imbalance.

\subsection*{Gradient Dynamics and Collapse Avoidance}

We analyze the gradient behavior of the contrastive loss to show that the proposed reweighting mechanism mitigates representation collapse and amplifies semantically informative updates.

\begin{proposition}[Gradient Amplification and Collapse Avoidance]
\label{prop:gradient_behavior}
Let \( \ell_u \) denote the contrastive loss for anchor node \( u \), and let \( \mathbf{h}_u \in \mathbb{R}^d \) be its embedding. Define:
\[
\ell_u := -\log \frac{
\sum\limits_{v \in \mathcal{S}^+(u)} \exp\left( \beta_{uv} \cdot \text{sim}(\mathbf{h}_u, \mathbf{h}_v)/\tau \right)
}{
\sum\limits_{v \in \mathcal{S}^+(u)} \exp\left( \beta_{uv} \cdot \text{sim}(\mathbf{h}_u, \mathbf{h}_v)/\tau \right) + \sum\limits_{w \in \mathcal{S}^-(u)} \exp\left( \beta_{uw} \cdot \text{sim}(\mathbf{h}_u, \mathbf{h}_w)/\tau \right)
}
\]
Assume similarity is defined as the dot product, and \( \beta_{uv} \), \( \beta_{uw} \) are defined as in Eq.~(xx) of the main paper. Then:

(i) As \( \text{sim}(\mathbf{h}_u, \mathbf{h}_w) \to 0 \), the gradient contribution from negative sample \( w \) vanishes unless \( \beta_{uw} \to \infty \);

(ii) The weight design guarantees that harder negatives (i.e., larger TBC difference) yield proportionally stronger gradients;

(iii) The overall gradient with respect to \( \mathbf{h}_u \) is bounded and nonzero as long as at least one hard negative exists.

\end{proposition}

\begin{proof}
Define the softmax-weighted probabilities:
\[
p_{uv}^{(+)} := \frac{\exp(\beta_{uv} \cdot \text{sim}(\mathbf{h}_u, \mathbf{h}_v)/\tau)}{Z_u}, \quad
p_{uw}^{(-)} := \frac{\exp(\beta_{uw} \cdot \text{sim}(\mathbf{h}_u, \mathbf{h}_w)/\tau)}{Z_u}
\]
where \( Z_u \) is the denominator partition term in \( \ell_u \). Then the gradient becomes:
\[
\nabla_{\mathbf{h}_u} \ell_u =
\frac{1}{\tau} \left[
\sum_{v \in \mathcal{S}^+(u)} \beta_{uv} \cdot p_{uv}^{(+)} (\mathbf{h}_u - \mathbf{h}_v)
-
\sum_{w \in \mathcal{S}^-(u)} \beta_{uw} \cdot p_{uw}^{(-)} (\mathbf{h}_u - \mathbf{h}_w)
\right]
\]

(i) If \( \text{sim}(\mathbf{h}_u, \mathbf{h}_w) \to 0 \), then \( \exp(\cdot) \to 1 \), so \( p_{uw}^{(-)} \to \frac{1}{|\mathcal{S}^-(u)|} \) and gradient magnitude is small. Without reweighting (\( \beta_{uw} = 1 \)), this leads to vanishing gradients.

(ii) Since \( \beta_{uw} \propto |TBC(u) - TBC(w)| \), larger semantic gaps yield larger weights, hence larger gradient contributions per sample.

(iii) The loss gradient remains bounded due to bounded dot products and finite-size neighborhoods, but is strictly nonzero if there exists at least one \( w \) such that \( \text{sim}(\mathbf{h}_u, \mathbf{h}_w) > 0 \) and \( \beta_{uw} > 1 \), which is ensured by your sampling definition.

Hence, the loss avoids degeneracy and maintains informative gradients aligned with the TBC structure.
\end{proof}

This result complements our semantic weighting design by formally showing that the gradient signal remains meaningful even under class imbalance or sparse hard negatives. It ensures training dynamics resist collapse and continue to differentiate node representations based on task-relevant cues.

\newpage
\section{Details of the Experiments}\label{appendix:datasets}

We evaluate our method on a diverse set of real-world temporal networks, encompassing human proximity, online interactions, and communication patterns. The datasets vary significantly in terms of scale, edge density, and temporal granularity, thereby testing the model’s robustness across heterogeneous dynamic structures. All datasets are publicly available and widely used in prior studies on temporal graph learning and dynamic centrality analysis. Table~\ref{tab:data} summarizes the key statistics of each dataset, including the number of nodes (\( |V| \)), edges (\( |E| \)), and distinct timestamps (\( |T| \)).

\begin{table}[htbp]
\centering
\small 
\setlength{\tabcolsep}{2.5pt} 
\renewcommand{\arraystretch}{1} 
\caption{Overview of datasets used in the experiment evaluation: $|V|$, $|E|$, and $|T|$ are the number of vertices, the number of edges, and the number of distinct timestamps, respectively.}
\label{tab:data}
\begin{tabular}{l r r r l r}
\toprule
\textbf{Dataset} & \textbf{$|V|$} & \textbf{$|E|$} & \textbf{$|T|$} & \textbf{Description} \\
\midrule
sp\_hospital~\cite{vanhems2013estimating}       & 75   & 32,424 & 9,453   & Face-to-face interactions in a hospital \\
sp\_hypertext~\cite{Isella:2011qo}     & 37   & 61,487 & 32,306  & Face-to-face interactions at conference  \\
sp\_workplace~\cite{genois2018can}      & 92   & 755  & 9,827   & Face-to-face interactions in a workspace \\
Highschool2011~\cite{fournet2014contact}    & 126  & 28,561 & 5,609   & High school social network 2011 \\
Highschool2012~\cite{fournet2014contact}    & 180  & 45,047 & 11,273  & High school social network 2012\\
Highschool2013~\cite{mastrandrea2015contact}     & 327  & 188,508 & 7,375  & High school social network 2013\\
haggle~\cite{chaintreau2007impact}            & 274  & 28,244 & 15,662  & Human proximity recorded by smart devices \\
ia-reality-call~\cite{eagle2006reality}   & 6,809 & 52,050  & 50,781 & Reality-based call dataset\\
infectious~\cite{Isella:2011qo}        & 10,972 & 415,912 & 76,944 & Infectious disease contact data\\
wikiedits-se~\cite{kunegis2013konect}      & 18,055 & 261,169 & 258,625 & Wikipedia edits (SE) dataset\\
sx-mathoverflow~ \cite{snap}   & 24,818 & 506,550 & 389,952 & Comments, questions, and answers on Math Overflow\\
superuser~\cite{snap}         & 194,085 & 1,443,339 & 1,437,199 & Comments, questions, and answers on Super User \\
\bottomrule
\end{tabular}
\end{table}

\textbf{Experiment Setup. }
For every dataset, we repeated each experiment 30 times, and we reported the mean and the standard deviation of all scores. 
We set the embedding dimensions of node features, continuous time, and temporal path count to 128, set the number of temporal neighbors sampled per node to 20, and employed a three-layer model trained over 15 epochs. To ensure fair comparison, we repeated all experiments for each deep learning model using three learning rates, 0.1, 0.01, and 0.001, and reported the best mean score. All experiments were conducted on a workstation equipped with an NVIDIA GeForce RTX 3090 GPU,  CUDA 12.2, and NVIDIA driver version 535.113.01.
Our training dataset consists of 50 real-world temporal networks from Network Repository~\cite{nr-temporal} and Konect~\cite{kunegis2013konect}. 

\section{Additional Results on Metrics and Model efficiency}\label{appendix:Results}

In the following, we provide additional experimental results. Specifically, we present a comparison of model efficiency between CLGNN and ETBC in terms of runtime (Figure~\ref{fig:clgnn_etbc_runtimes}),  and HitsIn10, 30, 50 performance comparison with static GNN baselines (Table~\ref{tab:hit_comparison_std}), temporal GNN baselines (Table~\ref{tab:temporal_hits}), contrastive variants (Table~\ref{tab:contrastive_hits_part1} and Table~\ref{tab:contrastive_hits_part2}), then we evaluate prediction accuracy across discretized TBC value ranges—including zero-value, mid-value, and high-value nodes—whose results (Table~\ref{tab:discretized_mae_ranges}).

\begin{figure*}[htbp]
\centering
\includegraphics[width=\textwidth]{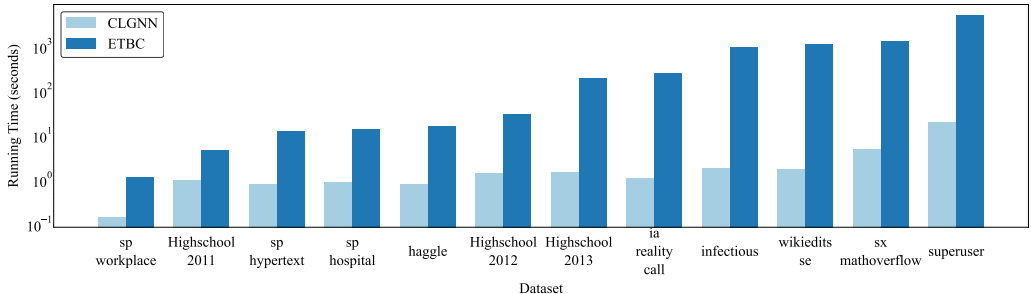}
\vspace*{-0.3in}
\caption{Model efficiency: comparison of CLGNN and ETBC}
\label{fig:clgnn_etbc_runtimes}
\end{figure*}

\begin{table*}[htbp]
\centering
\caption{HitsIn10/30/50  performance comparison with static GNN baselines}
\label{tab:hit_comparison_std}
\resizebox{\textwidth}{!}{
\begin{tabular}{l|ccc|ccc|ccc|ccc}
\toprule
\textbf{Dataset} & \multicolumn{3}{c|}{\textbf{CLGNN}} & \multicolumn{3}{c|}{GCN} & \multicolumn{3}{c|}{DrBC} & \multicolumn{3}{c}{GNN-Bet} \\
\cmidrule(r){2-4} \cmidrule(r){5-7} \cmidrule(r){8-10} \cmidrule(r){11-13}
& HitsIn10 & HitsIn30 & HitsIn50 & HitsIn10 & HitsIn30 & HitsIn50 & HitsIn10 & HitsIn30 & HitsIn50 & HitsIn10 & HitsIn30 & HitsIn50 \\
\midrule
sp\_hospital     & \textbf{3.0 $\pm$ 0.5} & \textbf{20.2 $\pm$ 0.3} & \textbf{46.2 $\pm$ 0.3} & 0.8 $\pm$ 0.6 & 9.2 $\pm$ 0.5 & 31.3 $\pm$ 0.5 & 0.9 $\pm$ 0.6 & 3.8 $\pm$ 0.7 & 17.7 $\pm$ 0.7 & 0.9 $\pm$ 0.6 & 3.7 $\pm$ 0.3 & 27.3 $\pm$ 0.4 \\
sp\_hypertext    & \textbf{2.8 $\pm$ 0.3} & \textbf{28.2 $\pm$ 0.4} & \textbf{37.6 $\pm$ 0.6} & 2.3 $\pm$ 0.7 & 26.7 $\pm$ 0.6 & 37.2 $\pm$ 0.3 & 0.4 $\pm$ 0.2 & 1.8 $\pm$ 0.4 & 27.2 $\pm$ 0.5 & 2.1 $\pm$ 0.6 & 24.9 $\pm$ 0.5 & 36.8 $\pm$ 0.6 \\
sp\_workplace    & \textbf{3.1 $\pm$ 0.5} & \textbf{10.2 $\pm$ 0.6} & \textbf{30.8 $\pm$ 0.4} & 0.6 $\pm$ 0.3 & 7.2 $\pm$ 0.4 & 28.8 $\pm$ 0.7 & 0.2 $\pm$ 0.2 & 0.9 $\pm$ 0.3 & 7.9 $\pm$ 0.6 & 0.1 $\pm$ 0.1 & 7.8 $\pm$ 0.4 & 18.8 $\pm$ 0.6 \\
Highschool2011   & \textbf{4.3 $\pm$ 0.7} & \textbf{11.2 $\pm$ 0.4} & \textbf{26.0 $\pm$ 0.4} & 3.9 $\pm$ 0.4 & 13.2 $\pm$ 0.4 & 21.0 $\pm$ 0.4 & 2.2 $\pm$ 0.4 & 1.8 $\pm$ 0.6 & 12.0 $\pm$ 0.4 & 0.4 $\pm$ 0.2 & 1.9 $\pm$ 0.5 & 10.7 $\pm$ 0.4 \\
Highschool2012 & \textbf{6.3 $\pm$ 0.5} & \textbf{9.2 $\pm$ 0.5} & \textbf{12.0 $\pm$ 0.3} & 0.8 $\pm$ 0.4 & 18.3 $\pm$ 0.5 & 29.3 $\pm$ 0.6 & 1.3 $\pm$ 0.6 & 1.1 $\pm$ 0.3 & 9.8 $\pm$ 0.4 & 1.1 $\pm$ 0.3 & 3.3 $\pm$ 0.3 & 10.2 $\pm$ 0.5 \\
Highschool2013 & \textbf{2.2 $\pm$ 0.6} & \textbf{4.1 $\pm$ 0.3} & \textbf{19.8 $\pm$ 0.6} & 0.0 $\pm$ 0.0 & 3.1 $\pm$ 0.5 & 9.1 $\pm$ 0.4 & 2.0 $\pm$ 0.5 & 3.0 $\pm$ 0.4 & 9.1 $\pm$ 0.6 & 0.0 $\pm$ 0.0 & 0.0 $\pm$ 0.0 & 2.2 $\pm$ 0.5 \\
haggle & \textbf{5.2 $\pm$ 0.4} & \textbf{27.2 $\pm$ 0.7} & \textbf{49.8 $\pm$ 0.2} & 0.0 $\pm$ 0.0 & 0.0 $\pm$ 0.0 & 0.0 $\pm$ 0.0 & 0.0 $\pm$ 0.0 & 0.0 $\pm$ 0.0 & 2.8 $\pm$ 0.5 & 0.1 $\pm$ 0.1 & 0.0 $\pm$ 0.0 & 2.9 $\pm$ 0.4 \\
ia-reality-call & \textbf{2.0 $\pm$ 0.5} & \textbf{13.9 $\pm$ 0.5} & \textbf{35.8 $\pm$ 0.6} & 0.1 $\pm$ 0.1 & 1.0 $\pm$ 0.6 & 2.9 $\pm$ 0.6 & 0.0 $\pm$ 0.0 & 0.0 $\pm$ 0.0 & 0.0 $\pm$ 0.0 & 0.0 $\pm$ 0.0 & 0.0 $\pm$ 0.0 & 0.0 $\pm$ 0.0 \\
infectious & \textbf{0.9 $\pm$ 0.5} & \textbf{4.0 $\pm$ 0.6} & \textbf{13.9 $\pm$ 0.7} & 0.0 $\pm$ 0.0 & 0.2 $\pm$ 0.2 & 0.3 $\pm$ 0.3 & 0.2 $\pm$ 0.2 & 0.3 $\pm$ 0.3 & 0.8 $\pm$ 0.6 & 0.3 $\pm$ 0.3 & 0.0 $\pm$ 0.0 & 0.0 $\pm$ 0.0 \\
wikiedits-se & \textbf{0.8 $\pm$ 0.4} & \textbf{5.1 $\pm$ 0.4} & \textbf{12.0 $\pm$ 0.3} & 0.0 $\pm$ 0.0 & 0.0 $\pm$ 0.0 & 0.0 $\pm$ 0.0 & 0.0 $\pm$ 0.0 & 0.0 $\pm$ 0.0 & 0.0 $\pm$ 0.0 & 0.2 $\pm$ 0.2 & 0.1 $\pm$ 0.1 & 0.0 $\pm$ 0.0 \\
sx-mathoverflow & \textbf{1.1 $\pm$ 0.3} & \textbf{3.2 $\pm$ 0.6} & \textbf{6.7 $\pm$ 0.6} & 0.0 $\pm$ 0.0 & 0.2 $\pm$ 0.2 & 0.0 $\pm$ 0.0 & 0.0 $\pm$ 0.0 & 0.0 $\pm$ 0.0 & 0.2 $\pm$ 0.2 & 0.1 $\pm$ 0.1 & 0.1 $\pm$ 0.1 & 0.0 $\pm$ 0.0 \\
superuser & \textbf{0.9 $\pm$ 0.4} & \textbf{7.0 $\pm$ 0.6} & \textbf{12.7 $\pm$ 0.4} & 0.0 $\pm$ 0.0 & 0.1 $\pm$ 0.1 & 0.0 $\pm$ 0.0 & 0.0 $\pm$ 0.0 & 0.0 $\pm$ 0.0 & 0.2 $\pm$ 0.2 & 0.3 $\pm$ 0.3 & 0.0 $\pm$ 0.0 & 0.0 $\pm$ 0.0 \\
\bottomrule
\end{tabular}
}
\vspace{-5pt}
\end{table*}

\begin{table*}[htbp]
\centering
\caption{HitsIn10/30/50 performance comparison with temporal GNN baselines}
\label{tab:temporal_hits}
\resizebox{\textwidth}{!}{
\begin{tabular}{l|ccc|ccc|ccc|ccc}
\toprule
\textbf{Dataset} & \multicolumn{3}{c|}{\textbf{CLGNN}} & \multicolumn{3}{c|}{\textbf{TGAT}} & \multicolumn{3}{c|}{\textbf{TATKC}} & \multicolumn{3}{c}{\textbf{DBGNN}} \\
\cmidrule(r){2-4} \cmidrule(r){5-7} \cmidrule(r){8-10} \cmidrule(r){11-13}
& HitsIn10 & HitsIn30 & HitsIn50 & HitsIn10 & HitsIn30 & HitsIn50 & HitsIn10 & HitsIn30 & HitsIn50 & HitsIn10 & HitsIn30 & HitsIn50 \\
\midrule
sp\_hospital & \textbf{3.0 $\pm$ 0.5} & \textbf{20.2 $\pm$ 0.5} & \textbf{46.2 $\pm$ 0.4} & 1.0 $\pm$ 0.3 & 9.0 $\pm$ 0.4 & 31.0 $\pm$ 0.7 & 0.0 $\pm$ 0.0 & 10.0 $\pm$ 0.5 & 34.0 $\pm$ 0.3 & 8.0 $\pm$ 0.6 & \textbf{24.0 $\pm$ 0.4} & 32.0 $\pm$ 0.4 \\
sp\_hypertext & \textbf{2.8 $\pm$ 0.7} & \textbf{28.2 $\pm$ 0.4} & \textbf{37.6 $\pm$ 0.6} & 2.0 $\pm$ 0.5 & 27.0 $\pm$ 0.6 & 37.0 $\pm$ 0.3 & 4.0 $\pm$ 0.6 & 27.0 $\pm$ 0.4 & 37.0 $\pm$ 0.4 & 6.0 $\pm$ 0.4 & 21.0 $\pm$ 0.5 & 34.0 $\pm$ 0.4 \\
sp\_workplace & \textbf{3.1 $\pm$ 0.6} & \textbf{10.2 $\pm$ 0.5} & \textbf{30.8 $\pm$ 0.7} & 0.0 $\pm$ 0.0 & 4.0 $\pm$ 0.7 & 19.0 $\pm$ 0.7 & 0.0 $\pm$ 0.0 & 9.0 $\pm$ 0.5 & 26.0 $\pm$ 0.6 & 4.0 $\pm$ 0.7 & 20.0 $\pm$ 0.7 & 30.0 $\pm$ 0.5 \\
Highschool2011 & \textbf{4.3 $\pm$ 0.4} & \textbf{11.2 $\pm$ 0.4} & \textbf{26.0 $\pm$ 0.5} & 0.0 $\pm$ 0.0 & 6.0 $\pm$ 0.3 & 25.0 $\pm$ 0.5 & 0.0 $\pm$ 0.0 & 11.0 $\pm$ 0.4 & 25.0 $\pm$ 0.4 & 2.0 $\pm$ 0.6 & 10.0 $\pm$ 0.6 & 22.0 $\pm$ 0.6 \\
Highschool2012 & \textbf{6.3 $\pm$ 0.4} & 9.2 $\pm$ 0.6 & 12.0 $\pm$ 0.5 & 0.0 $\pm$ 0.0 & 5.0 $\pm$ 0.5 & 14.0 $\pm$ 0.5 & 1.0 $\pm$ 0.4 & \textbf{10.0 $\pm$ 0.6} & \textbf{17.0 $\pm$ 0.6} & 3.0 $\pm$ 0.7 & 4.0 $\pm$ 0.3 & 11.0 $\pm$ 0.3 \\
Highschool2013 & 2.2 $\pm$ 0.6 & 4.1 $\pm$ 0.3 & \textbf{19.8 $\pm$ 0.6} & 0.0 $\pm$ 0.0 & 2.0 $\pm$ 0.6 & 5.0 $\pm$ 0.5 & 0.0 $\pm$ 0.0 & 5.0 $\pm$ 0.5 & 11.0 $\pm$ 0.5 & \textbf{3.0 $\pm$ 0.3} & \textbf{13.0 $\pm$ 0.3} & 17.0 $\pm$ 0.5 \\
haggle & \textbf{5.2 $\pm$ 0.4} & \textbf{27.2 $\pm$ 0.7} & \textbf{49.8 $\pm$ 0.2} & 0.0 $\pm$ 0.0 & 1.0 $\pm$ 0.3 & 18.0 $\pm$ 0.7 & 1.0 $\pm$ 0.4 & 8.0 $\pm$ 0.3 & 23.0 $\pm$ 0.6 & 5.0 $\pm$ 0.4 & 25.0 $\pm$ 0.4 & 37.0 $\pm$ 0.5 \\
ia\_reality\_call & \textbf{2.0 $\pm$ 0.5} & \textbf{13.9 $\pm$ 0.5} & \textbf{35.8 $\pm$ 0.6} & 0.0 $\pm$ 0.0 & 10.0 $\pm$ 0.2 & 13.0 $\pm$ 0.6 & 1.0 $\pm$ 0.4 & 10.0 $\pm$ 0.3 & 20.0 $\pm$ 0.5 & 1.0 $\pm$ 0.6 & 11.0 $\pm$ 0.5 & 30.0 $\pm$ 0.3 \\
infectious & \textbf{0.9 $\pm$ 0.5} & 4.0 $\pm$ 0.6 & 13.9 $\pm$ 0.7 & 0.0 $\pm$ 0.0 & 1.0 $\pm$ 0.5 & 4.0 $\pm$ 0.7 & 0.0 $\pm$ 0.0 & 2.0 $\pm$ 0.6 & 9.0 $\pm$ 0.2 & 0.5 $\pm$ 0.5 & \textbf{12.0 $\pm$ 0.4} & \textbf{15.0 $\pm$ 0.5} \\
wikiedits\_se & 0.8 $\pm$ 0.4 & \textbf{5.1 $\pm$ 0.4} & \textbf{12.0 $\pm$ 0.3} & 0.0 $\pm$ 0.0 & 5.0 $\pm$ 0.3 & 6.0 $\pm$ 0.6 & \textbf{1.0 $\pm$ 0.3} & 2.0 $\pm$ 0.4 & 5.0 $\pm$ 0.3 & 1.0 $\pm$ 0.6 & 3.0 $\pm$ 0.3 & 10.0 $\pm$ 0.5 \\
sx\_mathoverflow & \textbf{1.1 $\pm$ 0.3} & \textbf{3.2 $\pm$ 0.6} & \textbf{6.7 $\pm$ 0.6} & 1.0 $\pm$ 0.4 & 1.0 $\pm$ 0.6 & 2.0 $\pm$ 0.3 & 1.0 $\pm$ 0.5 & 2.0 $\pm$ 0.6 & 5.0 $\pm$ 0.4 & 1.0 $\pm$ 0.6 & 3.0 $\pm$ 0.6 & 3.0 $\pm$ 0.3 \\
superuser & \textbf{0.9 $\pm$ 0.4} & \textbf{7.0 $\pm$ 0.6} & \textbf{12.7 $\pm$ 0.4} & 1.0 $\pm$ 0.4 & 3.0 $\pm$ 0.5 & 11.0 $\pm$ 0.3 & 0.0 $\pm$ 0.0 & 2.0 $\pm$ 0.4 & 12.0 $\pm$ 0.6 & 0.0 $\pm$ 0.0 & 3.0 $\pm$ 0.3 & 13.0 $\pm$ 0.4 \\
\bottomrule
\end{tabular}
}
\vspace{-5pt}
\end{table*}

\begin{table*}[htbp]
\centering
\caption{HitsIn10/30/50 performance comparison with contrastive variants (Part I)}
\label{tab:contrastive_hits_part1}
\resizebox{\textwidth}{!}{%
\begin{tabular}{l|ccc|ccc|ccc}
\toprule
\textbf{Dataset} & \multicolumn{3}{c|}{\textbf{IIOIwD}} & \multicolumn{3}{c|}{\textbf{No Contrastive}} & \multicolumn{3}{c}{\textbf{Sampling}} \\
\cmidrule(r){2-4} \cmidrule(r){5-7} \cmidrule(r){8-10}
& HitsIn10 & HitsIn30 & HitsIn50 & HitsIn10 & HitsIn30 & HitsIn50 & HitsIn10 & HitsIn30 & HitsIn50 \\
\midrule
sp\_hospital & 0.0 $\pm$ 0.5 & 8.2 $\pm$ 0.4 & 28.7 $\pm$ 0.2 & 0.0 $\pm$ 0.5 & 8.1 $\pm$ 0.4 & 28.5 $\pm$ 0.5 & 0.3 $\pm$ 0.3 & 7.7 $\pm$ 0.3 & 28.8 $\pm$ 0.4 \\
sp\_hypertext & 2.1 $\pm$ 0.3 & 26.6 $\pm$ 0.5 & 27.5 $\pm$ 0.4 & 1.8 $\pm$ 0.2 & 27.2 $\pm$ 0.3 & 26.6 $\pm$ 0.3 & 1.5 $\pm$ 0.5 & 26.8 $\pm$ 0.4 & 26.8 $\pm$ 0.4 \\
sp\_workplace & 1.9 $\pm$ 0.3 & 10.3 $\pm$ 0.3 & 26.8 $\pm$ 0.4 & 1.6 $\pm$ 0.4 & 9.6 $\pm$ 0.5 & 27.3 $\pm$ 0.3 & 1.5 $\pm$ 0.4 & 10.2 $\pm$ 0.4 & 27.3 $\pm$ 0.2 \\
highschool2011 & 0.0 $\pm$ 0.4 & 0.3 $\pm$ 0.4 & 9.3 $\pm$ 0.3 & 0.0 $\pm$ 0.3 & 0.0 $\pm$ 0.2 & 8.5 $\pm$ 0.4 & 0.0 $\pm$ 0.4 & 0.4 $\pm$ 0.3 & 8.9 $\pm$ 0.4 \\
highschool2012 & 0.3 $\pm$ 0.5 & 0.8 $\pm$ 0.2 & 13.7 $\pm$ 0.3 & 0.3 $\pm$ 0.5 & 0.5 $\pm$ 0.4 & 13.9 $\pm$ 0.3 & 0.0 $\pm$ 0.3 & 1.4 $\pm$ 0.3 & 14.0 $\pm$ 0.4 \\
highschool2013 & 1.4 $\pm$ 0.3 & 2.6 $\pm$ 0.3 & 10.5 $\pm$ 0.3 & 1.2 $\pm$ 0.4 & 2.7 $\pm$ 0.4 & 9.9 $\pm$ 0.4 & 1.1 $\pm$ 0.4 & 2.6 $\pm$ 0.5 & 9.8 $\pm$ 0.3 \\
haggle & 0.0 $\pm$ 0.2 & 0.4 $\pm$ 0.5 & 0.0 $\pm$ 0.4 & 0.3 $\pm$ 0.4 & 0.0 $\pm$ 0.3 & 0.0 $\pm$ 0.5 & 0.4 $\pm$ 0.4 & 0.0 $\pm$ 0.3 & 0.2 $\pm$ 0.5 \\
ia\_reality\_call & 0.0 $\pm$ 0.4 & 3.2 $\pm$ 0.3 & 5.2 $\pm$ 0.3 & 0.0 $\pm$ 0.4 & 3.1 $\pm$ 0.5 & 5.2 $\pm$ 0.4 & 0.0 $\pm$ 0.3 & 2.8 $\pm$ 0.3 & 5.5 $\pm$ 0.3 \\
infectious & 0.4 $\pm$ 0.5 & 3.4 $\pm$ 0.3 & 10.5 $\pm$ 0.5 & 0.0 $\pm$ 0.5 & 3.5 $\pm$ 0.5 & 10.8 $\pm$ 0.3 & 0.4 $\pm$ 0.3 & 2.7 $\pm$ 0.4 & 11.4 $\pm$ 0.4 \\
wikiedits\_se & 0.3 $\pm$ 0.4 & 3.4 $\pm$ 0.5 & 6.0 $\pm$ 0.4 & 0.3 $\pm$ 0.4 & 3.2 $\pm$ 0.4 & 6.4 $\pm$ 0.3 & 0.0 $\pm$ 0.2 & 3.1 $\pm$ 0.2 & 6.0 $\pm$ 0.4 \\
sx\_mathoverflow & 0.0 $\pm$ 0.4 & 1.0 $\pm$ 0.4 & 1.2 $\pm$ 0.5 & 0.0 $\pm$ 0.3 & 1.3 $\pm$ 0.3 & 0.9 $\pm$ 0.2 & 0.0 $\pm$ 0.5 & 1.3 $\pm$ 0.4 & 0.9 $\pm$ 0.3 \\
superuser & 0.0 $\pm$ 0.4 & 0.5 $\pm$ 0.2 & 2.5 $\pm$ 0.2 & 0.4 $\pm$ 0.4 & 1.0 $\pm$ 0.2 & 3.0 $\pm$ 0.3 & 0.0 $\pm$ 0.3 & 0.9 $\pm$ 0.4 & 3.1 $\pm$ 0.2 \\
\bottomrule
\end{tabular}%
}
\vspace{-5pt}
\end{table*}

\begin{table*}[htbp]
\centering
\caption{HitsIn10/30/50 performance comparison with contrastive variants (Part II)}
\label{tab:contrastive_hits_part2}
\resizebox{\textwidth}{!}{%
\begin{tabular}{l|ccc|ccc}
\toprule
\textbf{Dataset} & \multicolumn{3}{c|}{\textbf{Tweedie Loss}} & \multicolumn{3}{c}{\textbf{Graph Aug}}  \\
\cmidrule(r){2-4} \cmidrule(r){5-7}
& HitsIn10 & HitsIn30 & HitsIn50 & HitsIn10 & HitsIn30 & HitsIn50 \\
\midrule
sp\_hospital & 0.0 $\pm$ 0.3 & 8.1 $\pm$ 0.2 & 28.8 $\pm$ 0.3 & 0.0 $\pm$ 0.4 & 7.7 $\pm$ 0.4 & 29.1 $\pm$ 0.2 \\
sp\_hypertext & 2.0 $\pm$ 0.3 & 27.5 $\pm$ 0.4 & 27.4 $\pm$ 0.5 & 2.1 $\pm$ 0.5 & 26.6 $\pm$ 0.3 & 26.5 $\pm$ 0.3 \\
sp\_workplace & 1.9 $\pm$ 0.2 & 10.4 $\pm$ 0.4 & 26.8 $\pm$ 0.2 & 1.8 $\pm$ 0.3 & 10.2 $\pm$ 0.4 & 27.4 $\pm$ 0.3 \\
highschool2011 & 0.0 $\pm$ 0.2 & 0.0 $\pm$ 0.2 & 9.4 $\pm$ 0.4 & 0.1 $\pm$ 0.5 & 0.3 $\pm$ 0.3 & 9.4 $\pm$ 0.4 \\
highschool2012 & 0.0 $\pm$ 0.5 & 1.5 $\pm$ 0.3 & 14.0 $\pm$ 0.3 & 0.0 $\pm$ 0.2 & 1.1 $\pm$ 0.4 & 13.6 $\pm$ 0.3 \\
highschool2013 & 0.5 $\pm$ 0.4 & 3.2 $\pm$ 0.2 & 10.0 $\pm$ 0.3 & 1.1 $\pm$ 0.3 & 3.2 $\pm$ 0.3 & 10.4 $\pm$ 0.2 \\
haggle & 0.4 $\pm$ 0.4 & 0.1 $\pm$ 0.2 & 0.0 $\pm$ 0.5 & 0.1 $\pm$ 0.2 & 0.0 $\pm$ 0.4 & 0.0 $\pm$ 0.2 \\
ia\_reality\_call & 0.4 $\pm$ 0.4 & 3.3 $\pm$ 0.4 & 5.1 $\pm$ 0.3 & 0.0 $\pm$ 0.4 & 2.8 $\pm$ 0.2 & 5.1 $\pm$ 0.3 \\
infectious & 0.1 $\pm$ 0.2 & 3.1 $\pm$ 0.5 & 10.6 $\pm$ 0.4 & 0.4 $\pm$ 0.4 & 3.2 $\pm$ 0.4 & 10.9 $\pm$ 0.3 \\
wikiedits\_se & 0.0 $\pm$ 0.4 & 2.5 $\pm$ 0.2 & 6.3 $\pm$ 0.3 & 0.0 $\pm$ 0.4 & 3.3 $\pm$ 0.3 & 6.1 $\pm$ 0.2 \\
sx\_mathoverflow & 0.0 $\pm$ 0.3 & 1.0 $\pm$ 0.4 & 1.2 $\pm$ 0.3 & 0.5 $\pm$ 0.4 & 1.1 $\pm$ 0.4 & 0.9 $\pm$ 0.3 \\
superuser & 0.0 $\pm$ 0.4 & 1.0 $\pm$ 0.5 & 3.2 $\pm$ 0.2 & 0.0 $\pm$ 0.4 & 0.5 $\pm$ 0.4 & 3.4 $\pm$ 0.4 \\
\bottomrule
\end{tabular}%
}
\vspace{-5pt}
\end{table*}

\begin{table*}[htbp]
\centering
\caption{Prediction MAE by CLGNN across discretized TBC ranges (zero/mid/high)}
\label{tab:discretized_mae_ranges}
\begin{tabular}{lccc}
\toprule
\textbf{Dataset} & \textbf{zero-value MAE} & \textbf{mid-value MAE} & \textbf{high-value MAE} \\
\midrule
sp\_hospital       & 2.04 $\pm$ 0.05 & 4.77 $\pm$ 0.12 & 6.81 $\pm$ 0.17 \\
sp\_hypertext      & 3.25 $\pm$ 0.06 & 14.89 $\pm$ 0.29 & 22.48 $\pm$ 0.44 \\
sp\_workplace      & 1.12 $\pm$ 0.02 & 4.48 $\pm$ 0.08 & 6.73 $\pm$ 0.11 \\
Highschool2011     & 6.35 $\pm$ 0.05 & 7.63 $\pm$ 0.06 & 12.71 $\pm$ 0.10 \\
Highschool2012     & 18.18 $\pm$ 0.15 & 54.53 $\pm$ 0.44 & 84.83 $\pm$ 0.69 \\
Highschool2013     & 26.77 $\pm$ 0.16 & 60.22 $\pm$ 0.37 & 73.60 $\pm$ 0.45 \\
haggle             & 0.67 $\pm$ 0.02 & 1.40 $\pm$ 0.03 & 3.15 $\pm$ 0.07 \\
ia-reality-call    & 25.53 $\pm$ 0.43 & 59.57 $\pm$ 1.01 & 85.09 $\pm$ 1.45 \\
infectious         & 15.28 $\pm$ 0.29 & 38.21 $\pm$ 0.73 & 53.44 $\pm$ 1.03 \\
wikiedits-se       & 2.08 $\pm$ 0.01 & 7.69 $\pm$ 0.04 & 6.77 $\pm$ 0.04 \\
sx-mathoverflow    & 20.47 $\pm$ 0.50 & 87.01 $\pm$ 2.12 & 156.96 $\pm$ 3.83 \\
superuser          & 43.14 $\pm$ 0.93 & 183.36 $\pm$ 3.97 & 330.77 $\pm$ 7.16 \\
\bottomrule
\end{tabular}
\vspace{-5pt}
\end{table*}

\newpage
\section{Additional Results on Hyperparameter Tuning}\label{appendix:Results_Hyper}

\begin{figure*}[htbp]
    \centering
    \begin{subfigure}[t]{0.5\linewidth}
        \centering
        \includegraphics[width=\linewidth]{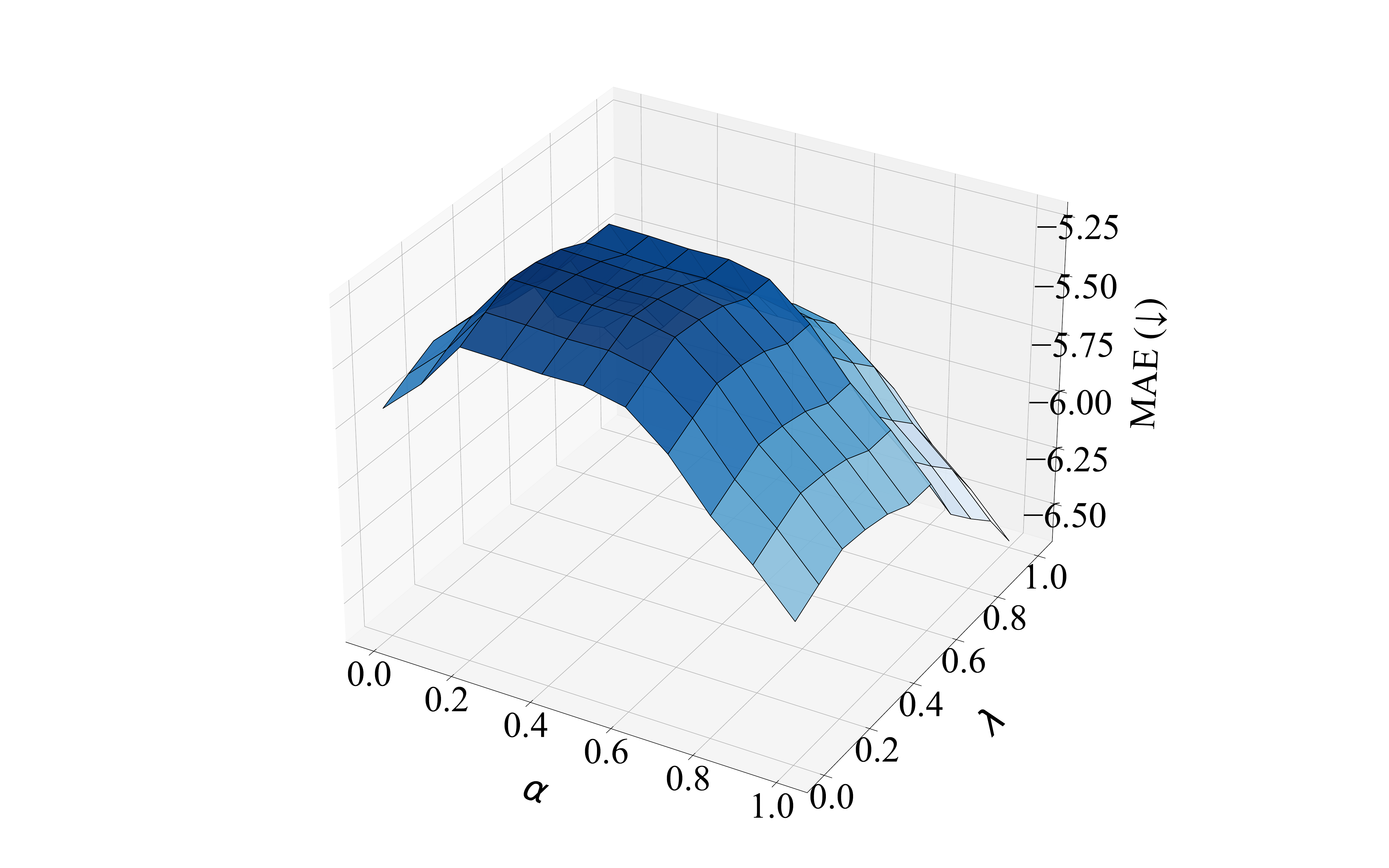}
        \caption{MAE surface over $(\alpha,\lambda)$}
        \label{fig:pingmian_mae}
    \end{subfigure}%
    \hfill
    \begin{subfigure}[t]{0.5\linewidth}
        \centering
        \includegraphics[width=\linewidth]{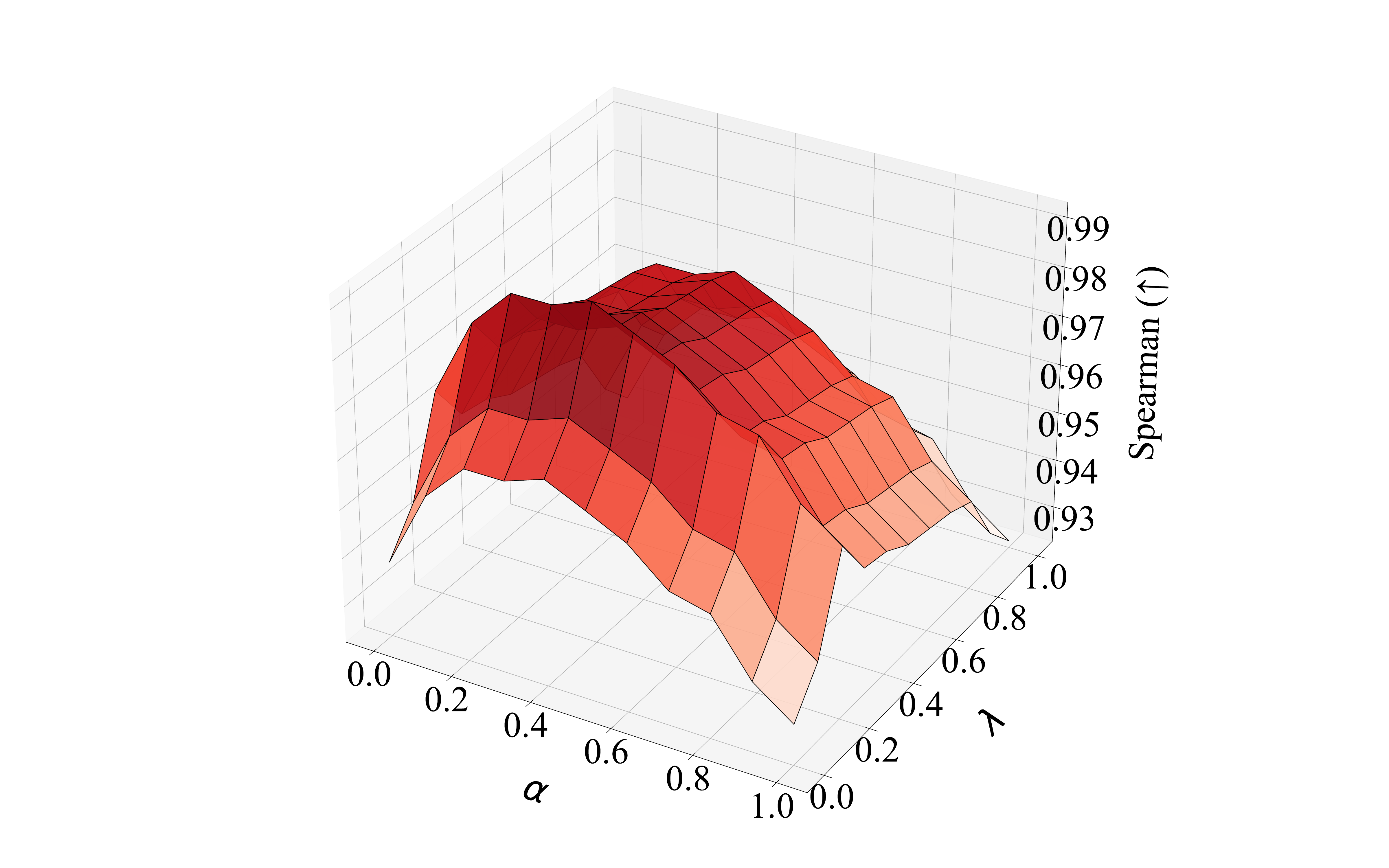}
        \caption{Spearman surface over $(\alpha,\lambda)$}
        \label{fig:pingmian_spe}
    \end{subfigure}

    \vspace{0.8em}

    \begin{subfigure}[t]{0.45\linewidth}
        \centering
        \includegraphics[width=\linewidth]{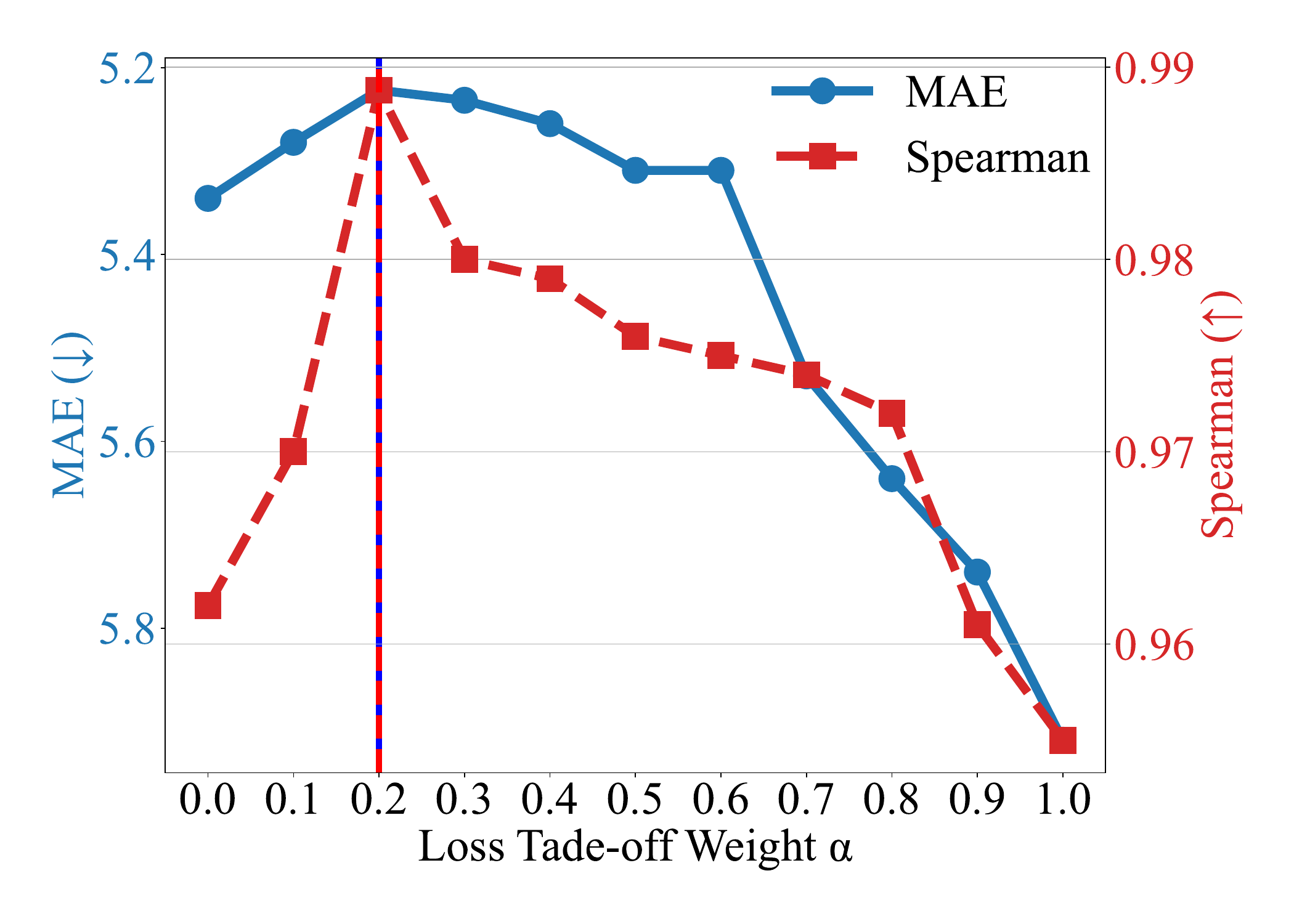}
        \caption{Tuning $\alpha$ (fix $\lambda=0.4$)}
        \label{fig:xiaorong_alpha}
    \end{subfigure}%
    \hfill
    \begin{subfigure}[t]{0.45\linewidth}
        \centering
        \includegraphics[width=\linewidth]{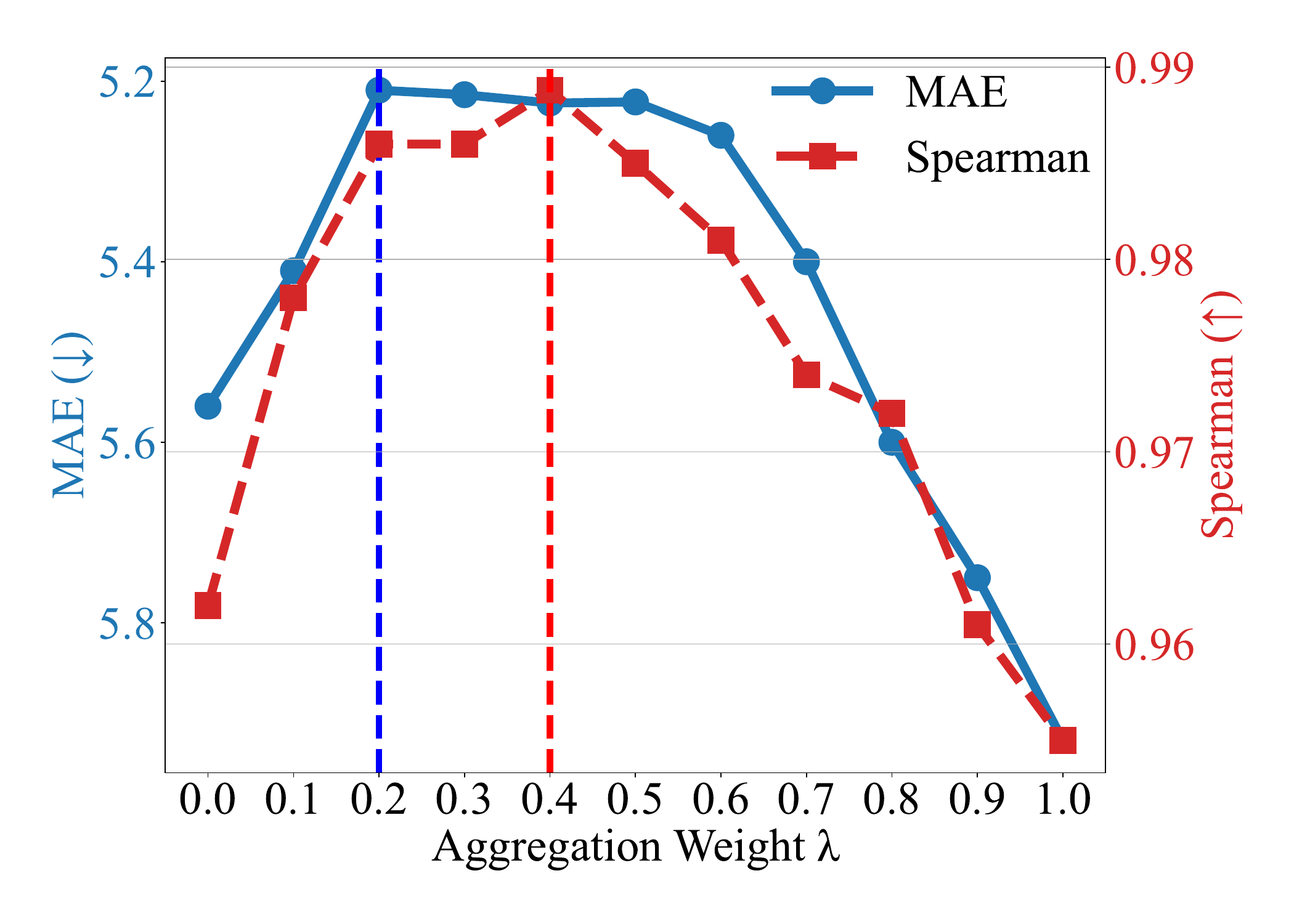}
        \caption{Tuning $\lambda$ (fix $\alpha=0.2$)}
        \label{fig:xiaorong_lambda}
    \end{subfigure}

    \caption{Hyperparameter tuning results over $(\alpha,\lambda)$ space and individual parameter ablations on the dataset haggle. (a)(b): joint performance surfaces for MAE and Spearman; (c)(d): 1D sensitivity analysis with fixed $\lambda$ and $\alpha$, respectively.}
    \label{fig:hyperparam_all}
    \vspace{0.5em}
    \begin{tcolorbox}[colback=gray!10!white, colframe=black!50, boxrule=0.5pt, arc=2pt, left=4pt, right=4pt, top=2pt, bottom=2pt, fonttitle=\bfseries, title=Note]
    The coefficient $\alpha$ controls the trade-off between the contrastive and regression objectives in the loss function. Introducing a moderate contrastive signal (e.g., $\alpha=0.2$) enhances the discriminability of node representations; however, overly large values may harm numerical accuracy.
    
    Meanwhile, $\lambda$ governs the neighborhood aggregation mechanism, balancing between mean aggregation (stable but coarse) and attention-based aggregation (informative but sensitive).
    
    Our tuning results suggest that \textit{contrastive learning should be weakly imposed, while attention should be moderately involved}, in order to achieve a representation that is both expressive and regression-friendly.
    \end{tcolorbox}

\end{figure*}

We conduct a comprehensive hyperparameter tuning study to investigate the effect of two key parameters in our framework: the loss trade-off coefficient $\alpha$, which balances the regression and contrastive objectives, and the aggregation weight $\lambda$, which controls the contribution between mean and attention-based neighborhood aggregation. Our analysis includes both 3D surface plots and fix-one-vary-one ablation curves on the dataset haggle.

As shown in Figure~\ref{fig:pingmian_mae} and Figure~\ref{fig:pingmian_spe}, we construct two performance surfaces over the $(\alpha, \lambda)$ space to evaluate MAE and Spearman correlation. Due to the inversion of the MAE axis for visualization clarity, the MAE surface appears convex, with its highest point corresponding to the lowest MAE, located near $(\alpha, \lambda) = (0.2, 0.2)$. In contrast, the Spearman surface reaches its global maximum near $(0.2, 0.4)$ without inversion. These results indicate that both hyperparameters jointly influence model performance, and their optimal settings are not independent. The alignment between low MAE and high Spearman regions further supports the effectiveness of the selected configuration.

To further isolate their individual impacts, we perform two 1D tuning experiments. First, we fix $\lambda=0.4$ and vary $\alpha$ (Figure~\ref{fig:xiaorong_alpha}). The results show that both MAE and Spearman improve as $\alpha$ increases from 0 to 0.2, with a performance peak (i.e., MAE trough) near $\alpha=0.2$. Beyond this point, performance gradually deteriorates. This trend stems from the nature of the loss design: the coefficient $\alpha$ governs the trade-off between the contrastive loss $\mathcal{L}_{\text{contrast}}$, which encourages representation-level separation, and the regression loss $\mathcal{L}_{\text{regress}}$, which directly optimizes numerical TBC values. A small portion of contrastive learning (e.g., $\alpha=0.2$) provides beneficial regularization, promoting structured embedding geometry and improving top-K ranking. However, when $\alpha$ becomes too large, the model overemphasizes clustering and neglects precise value regression, leading to worse MAE and flatter Spearman scores. In the inverted MAE surface plot (Figure~\ref{fig:pingmian_mae}), this manifests as a convex shape peaking at the optimal region.

Second, we fix $\alpha=0.2$ and vary $\lambda$, which governs the aggregation scheme in Eq.~(1). Here, $\lambda$ interpolates between mean aggregation $h_v^{\text{(mean)}}(t)$ and attention-based aggregation $h_v^{\text{(attention)}}(t)$. Mean aggregation operates at the node level, equally incorporating messages from all neighbors. It provides stability but lacks the ability to emphasize structurally important interactions. In contrast, attention aggregation is performed at the edge level, assigning dynamic weights to each neighbor based on temporal and structural context. This allows the model to focus on more informative time-respecting edges, which is particularly beneficial for estimating betweenness centrality in long-tailed, sparse temporal networks.

We observe that while $\lambda=0.2$ yields the lowest MAE, the performance difference from $\lambda=0.4$ is minor (5.21 vs. 5.2242), whereas Spearman correlation is notably higher at $\lambda=0.4$ (0.9888 vs. 0.986). Therefore, we adopt $\lambda=0.4$ as the default, prioritizing better ranking quality with only marginal loss in numerical accuracy. This confirms that moderate edge-level attention enhances discriminative power, while mean aggregation contributes to robust value estimation—and a balance of both yields the best trade-off.

\newpage
\section{Additional Results on Optimal Path Expansion}\label{appendix:Results_paths}

Table~\ref{tab:temporal_path_settings} investigates the impact of different temporal path definitions on model performance, including: (1) Shortest Path, (2) Shortest + Earliest Arrival (SEA), and (3) Shortest + Latest Departure + Earliest Arrival (SLDEA). As temporal constraints increase from left to right, we generally observe a degradation in prediction accuracy, ranking quality, and top-K identification—especially under the SLDEA setting, which requires modeling both forward and backward temporal dependencies. Some exceptions exist: datasets like sp\_hospital and sp\_hypertext show improved Spearman under SEA, likely because moderate temporal guidance (e.g., earliest arrival) filters out noisy paths and improves alignment with ground-truth centrality. Nevertheless, performance consistently drops under SLDEA, highlighting the difficulty for standard GNNs to capture complex temporal path semantics without explicit path-level modeling.

Quantitatively, models under the basic Shortest Path setting achieve the lowest MAE across 10 out of 13 datasets, showing 1.2 times to 2.4 times improvement in MAE compared to the most restrictive SLDEA variant. In terms of Spearman correlation, the Shortest Path definition leads to 1.3 times to 2.1 times higher values than SLDEA, and yields higher or equal Hits@10 in the majority of datasets. For instance, on Highschool2013, the Spearman drops from 0.6474 under Shortest Path to 0.2510 under SLDEA—a 2.6 times decrease in ranking quality. Similarly, on sp\_hypertext, the MAE increases from 6.09 to 50.15, an 8.2 times increase, highlighting the difficulty of modeling time-constrained paths.

These results suggest that while stricter path definitions such as SLDEA may better reflect real-world temporal dependencies, they introduce complex time-window constraints that standard GNNs struggle to encode. Most temporal GNNs aggregate over local neighborhoods without explicit modeling of time-respecting path sequences, making them less effective under restrictive temporal semantics. n contrast, the more relaxed Shortest Path and SEA definitions enable the model to benefit from a broader range of paths, yielding better stability and generalization. This analysis emphasizes the need for more advanced mechanisms—such as explicit time-aware path enumeration, reachability modeling, and sequence-level encoding—to effectively handle realistic but restrictive temporal path scenarios.

\begin{table*}[htbp]
\centering
\caption{Comparison under different temporal path definitions.}
\label{tab:temporal_path_settings}
\resizebox{\textwidth}{!}{
\begin{tabular}{l|ccc|ccc|ccc}
\toprule
\textbf{Dataset} & \multicolumn{3}{c|}{Shortest} & \multicolumn{3}{c|}{Shortest + Earliest Arrival (CLGNN)} & \multicolumn{3}{c}{Shortest + Latest Departure + Earliest Arrival} \\
\cmidrule(r){2-4} \cmidrule(r){5-7} \cmidrule(r){8-10}
& MAE & Spearman & hits@50 ($\uparrow$) & MAE & Spearman & hits@50 ($\uparrow$) & MAE & Spearman & hits@50 ($\uparrow$) \\
\midrule
ants-1-1 & \textbf{30.1114 $\pm$ 0.2445} & 0.3510 $\pm$ 0.0188 & 42 $\pm$ 4 & 49.0308 $\pm$ 0.6124 & \textbf{0.5585 $\pm$ 0.0164} & \textbf{43 $\pm$ 3} & 61.8830 $\pm$ 0.3835 & 0.3246 $\pm$ 0.0209 & 36 $\pm$ 1 \\
sp\_hospital & \textbf{7.9452 $\pm$ 0.1149} & 0.7213 $\pm$ 0.0124 & 42 $\pm$ 5 & 13.6171 $\pm$ 0.3270 & \textbf{0.7365 $\pm$ 0.0155} & \textbf{46 $\pm$ 4} & 25.4520 $\pm$ 0.2458 & 0.5123 $\pm$ 0.0128 & 35 $\pm$ 2 \\
sp\_hypertext & \textbf{6.0882 $\pm$ 0.0734} & 0.4890 $\pm$ 0.0199 & \textbf{37 $\pm$ 5} & 40.6243 $\pm$ 0.7979 & \textbf{0.4928 $\pm$ 0.0155} & 37 $\pm$ 4 & 50.1480 $\pm$ 0.5274 & 0.3379 $\pm$ 0.0260 & 36 $\pm$ 4 \\
sp\_workplace & \textbf{16.7165 $\pm$ 0.1769} & \textbf{0.5414 $\pm$ 0.0107} & \textbf{38 $\pm$ 1} & 12.3312 $\pm$ 0.2093 & 0.4538 $\pm$ 0.0191 & 31 $\pm$ 1 & 36.1800 $\pm$ 0.1632 & 0.3606 $\pm$ 0.0115 & 23 $\pm$ 1 \\
Highschool2011 & 33.4400 $\pm$ 0.1911 & \textbf{0.5617 $\pm$ 0.0282} & \textbf{35 $\pm$ 2} & \textbf{26.6913 $\pm$ 0.2167} & 0.5213 $\pm$ 0.0257 & 26 $\pm$ 3 & 37.3640 $\pm$ 0.3992 & 0.3976 $\pm$ 0.0297 & 26 $\pm$ 4 \\
Highschool2012 & \textbf{123.8627 $\pm$ 0.7080} & 0.4763 $\pm$ 0.0152 & \textbf{17 $\pm$ 2} & 157.5385 $\pm$ 1.2792 & \textbf{0.5057 $\pm$ 0.0230} & 12 $\pm$ 5 & 197.9800 $\pm$ 1.1633 & 0.3909 $\pm$ 0.0254 & 8 $\pm$ 3 \\
Highschool2013 & \textbf{85.1407 $\pm$ 0.3950} & \textbf{0.6474 $\pm$ 0.0233} & \textbf{28 $\pm$ 2} & 160.5927 $\pm$ 0.9895 & 0.5312 $\pm$ 0.0187 & 20 $\pm$ 3 & 192.1660 $\pm$ 0.9062 & 0.2510 $\pm$ 0.0140 & 16 $\pm$ 4 \\
haggle & \textbf{4.9951 $\pm$ 0.0676} & \textbf{0.9910 $\pm$ 0.0162} & 47 $\pm$ 3 & 5.2242 $\pm$ 0.1166 & 0.9888 $\pm$ 0.0284 & \textbf{49 $\pm$ 1} & 6.4640 $\pm$ 0.0933 & 0.7300 $\pm$ 0.0101 & 34 $\pm$ 5 \\
ia-reality-call & \textbf{149.0044 $\pm$ 1.5813} & \textbf{0.5076 $\pm$ 0.0204} & 23 $\pm$ 4 & 170.1932 $\pm$ 2.8971 & 0.4690 $\pm$ 0.0142 & \textbf{36 $\pm$ 1} & 186.9010 $\pm$ 2.7329 & 0.3009 $\pm$ 0.0263 & 24 $\pm$ 2 \\
infectious & \textbf{89.2501 $\pm$ 1.0522} & \textbf{0.5890 $\pm$ 0.0209} & 13 $\pm$ 5 & 106.9288 $\pm$ 2.0489 & 0.5503 $\pm$ 0.0188 & \textbf{14 $\pm$ 2} & 100.5650 $\pm$ 1.2965 & 0.3873 $\pm$ 0.0241 & 12 $\pm$ 4 \\
wikiedits-se & \textbf{10.6679 $\pm$ 0.0451} & \textbf{0.5313 $\pm$ 0.0137} & 12 $\pm$ 3 & 16.5407 $\pm$ 0.0895 & 0.5124 $\pm$ 0.0210 & \textbf{12 $\pm$ 4} & 25.8910 $\pm$ 0.1903 & 0.4142 $\pm$ 0.0246 & 9 $\pm$ 4 \\
sx-mathoverflow & \textbf{227.0346 $\pm$ 3.3304} & \textbf{0.2310 $\pm$ 0.0294} & 4 $\pm$ 3 & 264.4417 $\pm$ 6.4519 & 0.2140 $\pm$ 0.0237 & \textbf{7 $\pm$ 1} & 311.0830 $\pm$ 1.5786 & 0.1815 $\pm$ 0.0254 & 4 $\pm$ 4 \\
superuser & \textbf{534.7802 $\pm$ 7.0360} & \textbf{0.4423 $\pm$ 0.0255} & \textbf{15 $\pm$ 3} & 557.2800 $\pm$ 12.0645 & 0.4140 $\pm$ 0.0336 & 13 $\pm$ 4 & 592.9630 $\pm$ 6.8348 & 0.2512 $\pm$ 0.0115 & 4 $\pm$ 3 \\
\bottomrule
\end{tabular}
}
\end{table*}

\newpage

\end{document}